\pdfoutput=1
\documentclass{article}

\usepackage{iclr2027_conference,times}

\usepackage[T1]{fontenc}
\usepackage{amsmath}
\usepackage{amssymb}
\usepackage{amsfonts}
\usepackage{amsthm}
\usepackage{nicefrac}
\usepackage{microtype}
\usepackage{bm}
\usepackage{enumitem}
\usepackage{multirow}
\usepackage{algorithm}
\usepackage{algorithmic}
\usepackage{booktabs}
\usepackage{colortbl}
\usepackage{xcolor}
\usepackage{graphicx}
\usepackage{hyperref}
\usepackage{url}
\hypersetup{
  pdftitle={When More Foundation Models Means Less: Diagnosing and Addressing Multi-View Fusion Failure},
  pdfauthor={Yibo Liu and Bowen Jiang},
  pdfkeywords={foundation models, multi-view learning, encoder selection, kernel alignment},
  colorlinks=true,
  linkcolor=black,
  citecolor=black,
  urlcolor=black
}

\newtheorem{theorem}{Theorem}
\newtheorem{proposition}{Proposition}
\newtheorem{corollary}{Corollary}
\newtheorem{lemma}{Lemma}
\newtheorem{assumption}{Assumption}

\theoremstyle{definition}
\newtheorem{definition}{Definition}

\theoremstyle{remark}

\newcommand{\Ms}{\mathcal{M}}
\newcommand{\Xs}{\mathcal{X}}
\newcommand{\Ys}{\mathcal{Y}}
\newcommand{\Ps}{P_{\mathrm{S}}}
\newcommand{\Pt}{P_{\mathrm{T}}}
\newcommand{\muS}{\mu_{\mathrm{S}}}
\newcommand{\muT}{\mu_{\mathrm{T}}}
\newcommand{\hmuS}{\hat\mu_{\mathrm{S}}}
\newcommand{\hmuT}{\hat\mu_{\mathrm{T}}}
\newcommand{\Wone}{W_1}
\newcommand{\hpi}{\hat\pi}
\newcommand{\hk}{\hat k^\star}
\newcommand{\Ldelta}{L_\delta}
\newcommand{\barR}{\bar R}
\newcommand{\barD}{\bar D}
\newcommand{\barC}{\bar C}
\newcommand{\Dstop}{\Delta^{\mathrm{stop}}}
\newcommand{\DS}{\Delta_{\mathrm{S}}}
\newcommand{\DT}{\Delta_{\mathrm{T}}}
\newcommand{\sd}[1]{{\scriptsize$\pm$#1}}

\def\eg{\emph{e.g.}}
\def\Eg{\emph{E.g.}}
\def\etal{\emph{et al.}}

\usepackage{xr-hyper}

\title{
When More Foundation Models Means Less: Diagnosing and Addressing Multi-View Fusion Failure
}

\author{
Yibo Liu\thanks{Corresponding author.} \quad Bowen Jiang \\
Hainan Campus, Beijing University of Posts and Telecommunications \\
Lingshui, Hainan, China \\
\texttt{liuyibo@bupt.edu.cn} \quad \texttt{1729773236@bupt.edu.cn}
}

\iclrfinalcopy

\let\originalpaperlabel\label
\newcommand{\mainpaperlabel}[1]{%
  \originalpaperlabel{#1}\originalpaperlabel{main-#1}}
\newcommand{\supppaperlabel}[1]{%
  \originalpaperlabel{#1}\originalpaperlabel{supp-#1}}

\begin{document}

\maketitle
\lhead{Preprint}
\let\label\mainpaperlabel

\begin{abstract}
\textcolor{black}{Multi-view learning is motivated by the assumption that complementary views can provide richer and more discriminative representations. With the growing availability of pretrained foundation-model encoders, candidate views are no longer limited to a small, fixed set of sensors or hand-crafted features. In this setting, each frozen pretrained encoder can induce a distinct representation view, and practitioners may choose from a large, heterogeneous pool of candidate encoders. The design problem therefore shifts from fusing a small, predefined set of views to selecting which encoders to combine, and how many, from a large candidate pool. Empirically, we find that downstream performance does not improve monotonically as more encoders are fused: gains typically saturate after a small subset of encoders is included. These observations motivate explicitly selecting both the composition and the size of the fused view set. We formulate this problem as view-set composition: jointly determining which candidate views to fuse and how many to include. We then propose KAGES (Kernel-Alignment Greedy Encoder Selector), a label-aware selector that requires no downstream classifier training during selection and greedily adds the encoder yielding the largest marginal improvement in alignment between the fused representation kernel and the label kernel.  KAGES consistently outperforms competing selection methods as well as full fusion, and approaches exhaustive oracle selection. We further observe similar selection gains in image retrieval and in the fusion of frozen language-model representations. Overall, our results show that effective multi-view learning over foundation-model pools depends not on fusing more encoders, but on selecting a compact and complementary set of views.}
Code and results are available at \url{https://github.com/yibol9768-alt/Quantifying-Representation-Reliability}.
\end{abstract}

\section{Introduction}
\label{sec:intro}

\begin{figure}[t]
\centering
\includegraphics[width=\textwidth]{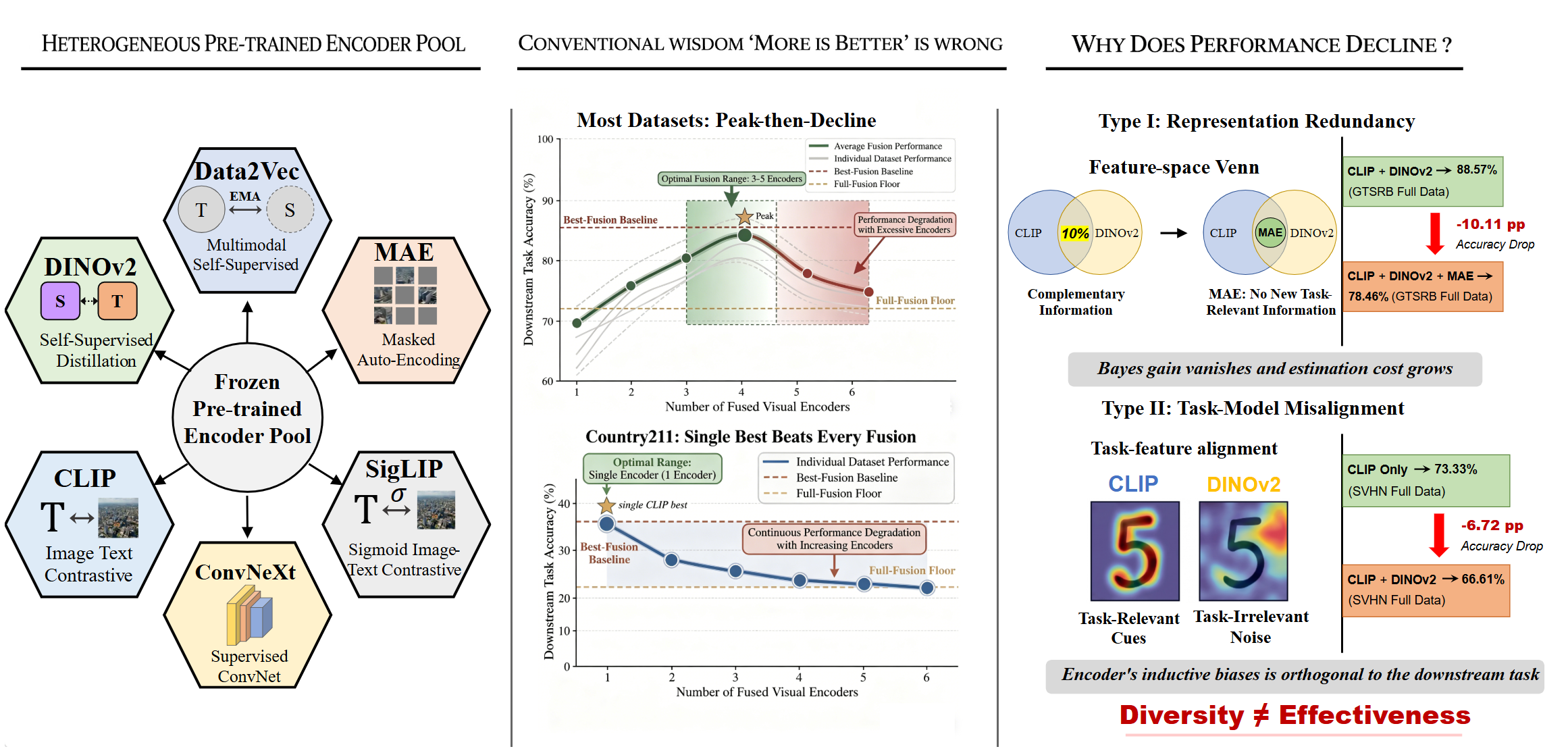}
\caption{\textbf{View-set composition in frozen foundation-model fusion.} \textbf{Left:} a heterogeneous pool of frozen vision backbones. \textbf{Middle:} downstream accuracy is non-monotonic in the number of fused views, with most datasets peaking at $3$ to $5$ encoders and declining, while a single task-aligned encoder beats every fusion on Country211. \textbf{Right:} two failure modes behind the decline. Type I (\emph{representation redundancy}): a new encoder repeats information already captured. Type II (\emph{task-model misalignment}): a representationally distinct encoder injects task-irrelevant cues, so diversity alone is not effectiveness. The marginal-utility theory of Section~\ref{sec:peak-theory} explains the shape from a single gain-cost crossing (Theorem~\ref{thm:peak}); the two failure modes are formalised in supplementary App.~\ref{supp-app-y:proof-failure-modes}.}
\label{fig:overview}
\originalpaperlabel{main-fig:overview}
\end{figure}

\textcolor{black}{Multi-view learning has traditionally assumed a small, predefined set of views arising from different sensors, modalities, or feature representations, motivated by the premise that complementary views can provide information beyond that available from any single view~\citep{xu2013multi,li2019mvlearn,trosten2023mvclust}. Foundation models fundamentally expand the space of candidate views. In the frozen-encoder setting, each pretrained encoder induces a distinct representation of the same input, and practitioners can draw from large, heterogeneous model pools spanning architectures, pretraining objectives, supervision sources, and modalities~\citep{radford2021learning,oquab2023dinov2,he2022masked,liu2022convnet,dubey2024llama3,yang2024qwen25}. The view set therefore ceases to be merely a fixed input to the learning problem and becomes a design variable itself: beyond deciding \emph{how} to fuse a predefined set of views, one must determine \emph{which} encoder-induced views to combine and \emph{how many} to retain. We formulate this problem as \emph{view-set composition}. While complementarity motivates combining multiple views, it does not by itself provide an operational criterion for selecting a subset from a large heterogeneous pool or for determining when additional views cease to be beneficial.
}

\textcolor{black}{Recent work on multi-encoder vision systems, including Eagle~\citep{shi2024eagle}, Cambrian-1~\citep{tong2024cambrian}, BRAVE~\citep{kar2024brave}, and MoVA~\citep{zong2024mova}, has demonstrated the value of combining representations from multiple pretrained encoders, with some methods further introducing routing or search mechanisms to determine how encoder features are used. These methods typically focus on designing effective fusion or routing strategies over a fixed or manually specified set of encoders, leaving view-set composition over large heterogeneous encoder pools largely unexplored.}

{\color{black}The importance of view-set composition becomes apparent as the encoder pool grows.
Even along a deliberately favorable fusion path that places previously successful and complementary encoders early, downstream accuracy is often non-monotonic as additional views are added. Encoders added early in the fusion path can improve performance, whereas later encoders may contribute information already captured by the current set or signals that are only weakly aligned with the target task. In either case, the fused representation gains additional view-specific information without a commensurate increase in task-relevant signal, so adding more encoders can ultimately reduce rather than improve downstream accuracy.

This observation does not imply that effective fusion requires a small fixed number of encoders. Rather, it shows that the peak observed along a fusion path is order-dependent. A fixed encoder-addition path may include informative views early and redundant or task-misaligned views later; a task-aware selector can instead prioritize high-value views while deferring less beneficial additions. We therefore cast view-set composition as an \emph{ordering problem}: a high-quality ordering should prioritize useful encoders so that its prefixes provide competitive fused sets under different compute budgets.

Existing selection strategies are insufficient for this goal. Per-encoder transferability scores~\citep{you2021logme,bao2019hscore,pandy2022gbc} evaluate individual encoders independently and ignore their interactions with already-selected encoders. Diversity-based criteria encourage representational difference, but \emph{different} is not necessarily \emph{useful} for the downstream task. Exhaustive or validation-based subset search can identify high-performing combinations, but requires training downstream predictors for many candidate subsets and does not scale well to large encoder pools. What is needed is a training-free, set-level criterion that scores the marginal utility of each candidate relative to the current fused set.

To address this challenge, we propose \emph{KAGES} (Kernel-Alignment Greedy Encoder Selector), a label-aware method for view-set composition over frozen foundation-model encoders. KAGES defines the utility of a fused set as the alignment between its representation kernel and the label kernel. At each step, it adds the candidate encoder with the largest marginal gain in this set-level utility. The resulting procedure produces a task-aligned encoder ordering without requiring downstream classifier training, combinatorial subset search, or fusion-architecture tuning.

The kernel-alignment objective admits a natural relevance--redundancy interpretation. In particular, the marginal gain favors candidates that are well aligned with the labels while exhibiting limited kernel overlap with already-selected encoders. Under linear-kernel fusion by feature concatenation, the fused kernel is additive across views, allowing KAGES to evaluate marginal gains directly at the kernel level rather than retraining a downstream model for each candidate subset (Section~\ref{sec:ker-evid}). The same set-level formulation also admits a conditional fixed-budget guarantee: if the pure alignment utility is monotone and has submodularity ratio at least $\gamma$, then, for any budget $k$, the greedy prefix of size $k$ achieves a $(1 - e^{-\gamma})$ approximation to the optimal size-$k$ subset (Theorem~\ref{thm:y:weak-submod}, Corollary~\ref{cor:surrogate-peak}).

Extensive evaluations on image-classification benchmarks show that KAGES consistently outperforms competing multi-view selection methods and full fusion, substantially improving over diversity-based selection while approaching oracle selection. These gains extend beyond classification to image retrieval and frozen large-language-model fusion.
\paragraph{Contributions.}
\begin{itemize}[nosep,leftmargin=*]

  \item \textbf{A new framing.}
  We recast large-pool frozen foundation-model fusion as \emph{multi-view learning at scale} and formulate \emph{view-set composition}---which encoders to fuse and how many---as a central design problem.

  \item \textbf{A unified characterisation, with theory.}
  We characterise the failure of the more-is-better assumption in this regime: a gain--cost crossing (Theorem~\ref{thm:peak}) provides a mechanism for an order-dependent peak, while empirical downstream performance subsequently plateaus and, on structured-symbol tasks, can decline even under an oracle ordering.

  \item \textbf{A scalable, set-level selection method.}
  We propose \emph{KAGES}, a label-aware selector that evaluates candidate encoders through their marginal contribution to a set-level kernel-alignment objective, avoiding repeated fitting of downstream predictors for candidate subsets. After per-encoder kernel construction, each candidate is evaluated in $O(n^2)$ time per greedy step, independent of encoder dimensionality. Under the stated monotonicity and submodularity-ratio conditions, each size-$k$ greedy prefix achieves a conditional $(1-e^{-\gamma})$ approximation to the optimal size-$k$ subset. KAGES remains effective in our largest evaluated pools of $M=15$ candidates.

  \item \textbf{Broad empirical validation.}
  We validate KAGES across image classification, image retrieval, and frozen large-language-model fusion, and observe the peak-then-decline phenomenon in both image classification and frozen-LLM fusion. Across these settings, KAGES produces stronger prefixes than diversity-based selection and often outperforms full fusion, with the largest gains in the main classification protocols, while approaching oracle selection.
\end{itemize}
}
\section{Peak-then-Decline and the Need for View-Set Composition}
\label{sec:phenomenon}

We now make the failure of complementarity precise. Once redundancy and finite-sample cost overtake the fresh task signal a new encoder brings along an ordering, every further view lowers the accuracy along that path. We first cast large-pool fusion as the search for a high-quality ordering (Section~\ref{sec:problem}), then give a marginal-utility account in which a single gain--cost crossing forces an order-dependent peak (Section~\ref{sec:peak-theory}, Theorem~\ref{thm:peak}). We then document the resulting peak-then-decline across five recognition regimes and three sample sizes on a path deliberately built to favour the ``more is better'' hypothesis (Sections~\ref{sec:diag-setup}--\ref{sec:why-decline}); the path is not our method but evidence that the peak is structural and leaves a large improvement open, which KAGES realises in Section~\ref{sec:method}.

\subsection{Large-pool fusion as an ordered path}
\label{sec:problem}

Let $\Ms = \{m_1, \ldots, m_M\}$ be a fixed pool of frozen pre-trained encoders. For input $x$, encoder $m_i$ produces a feature vector $F_{m_i}(x) \in \mathbb{R}^{d_{m_i}}$. Given a downstream classification task with label $Y$, a subset $S \subseteq \Ms$ is combined into a fused feature $F_S(x)$ by a fusion operator (defined in Section~\ref{sec:diag-setup}); a lightweight downstream classifier $f_\theta$ is then trained on a labelled set of size $n$. For the theoretical analysis below, let
\[
\bar F_S(x) := \bigl(F_m(x)\bigr)_{m \in S}
\]
denote the joint collection of the selected encoder features.

An ordering $\pi = (\pi_1, \ldots, \pi_K)$ of $K \le M$ encoders induces nested prefixes
\begin{equation}
\label{eq:nested-prefix}
S_k^\pi \;=\; \{m_{\pi_1}, \ldots, m_{\pi_k}\},
\qquad k = 1, \ldots, K,
\qquad S_0^\pi := \varnothing.
\end{equation}
For this ordering, the peak accuracy along the path is
\begin{equation}
\label{eq:peak-defs}
P_n(\pi) \;=\; \max_{1 \le k \le K}\,\mathrm{Acc}_n\bigl(S_k^\pi\bigr).
\end{equation}
The peak height $P_n(\pi)$ is a property of the ordering, not only of the underlying set: two permutations of the same encoders induce different nested prefixes and can therefore attain different peaks. Writing $R_n(S)$ for the corresponding $0$--$1$ test risk, $\mathrm{Acc}_n(S)=1-R_n(S)$, so maximising peak accuracy is equivalent to minimising the risk along the path. Section~\ref{sec:peak-theory} writes this minimum path risk in terms of a maximum prefix sum of conditional net marginal utilities along $\pi$ (Proposition~\ref{prop:peakgain}). View-set composition is the problem of constructing an ordering whose prefix accuracy curve $k \mapsto \mathrm{Acc}_n(S_k^\pi)$ stays as high as possible at every cardinality.

\subsection{A marginal-utility view of peak-then-decline}
\label{sec:peak-theory}

For a fixed ordering $\pi$, the performance peak is determined by the sequence of marginal gains and costs along its prefixes. Writing $S_k = S_k^\pi$ and $m_k := m_{\pi_k}$ for brevity, so that $S_k = S_{k-1} \cup \{m_k\}$, decompose the finite-sample test risk into Bayes and excess-risk parts,
\begin{equation}
\label{eq:risk-decomp}
R_n(S) = R^\star(S) + E_n(S),\quad
g_k := R^\star(S_{k-1}) - R^\star(S_k) \ge 0,\quad
c_k := E_n(S_k) - E_n(S_{k-1}),
\end{equation}
where $R^\star(S)$ is the Bayes risk given the joint representation $\bar F_S$, and $E_n(S)$ is the corresponding excess-risk term. Subtraction gives the one-step identity
\[
R_n(S_k) - R_n(S_{k-1}) = c_k - g_k:
\]
adding encoder $m_k$ improves performance exactly when $g_k > c_k$ (supplementary App.~\ref{supp-app-y:proof-risk}). Under log loss, the Bayes gain is a conditional mutual information,
\begin{equation}
\label{eq:gk-cmi}
g_{k,\log} \;=\; I(F_{m_k};\, Y \mid \bar F_{S_{k-1}}),
\end{equation}
the information-theoretic counterpart of the algebraic expansion in Proposition~\ref{prop:expansion}. This identity shows why ordering matters. The same candidate can have high gain when added before redundant views and low gain after its information has already been covered. A task-aware selector therefore raises the attainable peak by arranging the path so that large positive marginal utilities appear early.

The Bayes gain is non-increasing along a submodular greedy path (supplementary App.~\ref{supp-app-y:proof-greedy}). An OLS linear-probe instance gives
\[
\mathbb{E}\, R_n(S_k)
=
\nu_k^2 \frac{n-1}{n-k-1};
\]
once all signal coordinates have been included, the resulting pure-noise tail has non-decreasing excess-risk increments $c_k$ at fixed $n$ (supplementary App.~\ref{supp-app-y:proof-peak}, Step~4). These observations motivate the monotonicity assumptions in the following conditional single-crossing result.

\begin{theorem}[Order-specific gain--cost crossing]
\label{thm:peak}
Fix an ordering $\pi$ and its nested prefixes $S_1^\pi, \ldots, S_K^\pi$. Suppose $g_k$ is non-increasing along this path, $c_k$ is non-decreasing, and $a_k := g_k - c_k$ satisfies $a_1 > 0$, $a_K < 0$, with no zero-margin ties (margin $\eta := \min_k |a_k| > 0$). Let
\[k^\star(\pi) := \max\{k : a_k > 0\}.\]
Then $a_k$ changes sign exactly once, between $k^\star(\pi)$ and $k^\star(\pi)+1$, and, by the one-step identity,
\[R_n(S_0^\pi)>R_n(S_1^\pi)>
\cdots>R_n(S_{k^\star(\pi)}^\pi)<R_n(S_{k^\star(\pi)+1}^\pi)<\cdots<R_n(S_K^\pi).\]
Thus, risk decreases strictly up to $k^\star(\pi)$ and increases strictly afterwards. The peak value attained by this path is order-specific: changing the ordering changes the sequence $\{g_k, c_k\}$, and can therefore change both the peak location $k^\star(\pi)$ and the peak height $P_n(\pi)$.
\end{theorem}

Selection acts on the \emph{entire} marginal-utility sequence: a better ordering front-loads larger positive net marginal utilities and can achieve a higher peak before those positive net utilities are exhausted. Writing the path peak as a maximum prefix sum makes this precise.

\begin{proposition}[Peak lift by front-loading net utility]
\label{prop:peakgain}
For any ordering $\pi$, let $a_k^\pi := g_k^\pi - c_k^\pi$. Telescoping the one-step identity from $i = 1$ to $k$ yields
\[\min_{1 \le k \le K} R_n(S_k^\pi)=R_n(S_0^\pi)-\max_{1 \le k \le K}\sum_{i=1}^{k}a_i^\pi.\]
Consequently, an ordering $\pi'$ with a larger max-prefix-sum of net marginal utilities attains a no-worse peak risk than $\pi$.
\end{proposition}

\emph{Proof.}
Sum the one-step identity to obtain
\[R_n(S_k^\pi)=R_n(S_0^\pi)
-\sum_{i=1}^{k} a_i^\pi;\]
the minimum risk along $\pi$ is at the index maximising the cumulative net utility.
\hfill$\square$

The information form
\[g_{k,\log}=I(F_{m_k}; Y \mid \bar F_{S_{k-1}})\]
interprets each conditional gain as task signal not yet covered by the current set; the one-step identity itself applies to any decomposable risk, including $0$--$1$ loss with a different $g_k$. Under the single-crossing conditions above, if $c_k(n) = \alpha_k / n$ with $\alpha_k \ge 0$, increasing $n$ weakly decreases every $c_k$, so the peak index
\[\max\{k : g_k > c_k(n)\}\]
is non-decreasing in $n$; the next subsections show the peaked regime is the rule and that the peak shifts right with $n$ on most datasets.

\subsection{A favorable diagnostic path}
\label{sec:diag-setup}

Two taxonomies recur throughout the paper. \textbf{Encoder paradigm} partitions pre-training objectives into five classes: (1) language-supervised (CLIP, SigLIP), (2) self-supervised (DINOv2, MAE, Data2Vec), (3) generative or image-prior (Stable-Diffusion VAE encoder), (4) class-supervised (ConvNeXt, ResNet-50, ViT-B/16), and (5) task-expert (SAM, Depth-Anything). \textbf{Recognition regime} partitions downstream classification tasks into five classes: generic object, fine-grained object, texture, geo-spatial scene, and structured symbol or digit. Supplementary App.~\ref{supp-app-y:config} gives the full taxonomy and per-paradigm encoder rationales.

The diagnostic pool contains six frozen encoders. Cambrian-1~\citep{tong2024cambrian} reports that the trio CLIP-ViT-B/16~\citep{radford2021learning}, DINOv2-base~\citep{oquab2023dinov2}, and ConvNeXt-base~\citep{liu2022convnet} delivers a $+3$ to $+5$ point improvement over single-encoder baselines on its vision-centric benchmark CV-Bench, and treats this trio as the main fusion combination of the paper. We adopt it as the first three encoders of our pool. We then add a second representative of each of the three paradigms it spans: SigLIP-B/16~\citep{zhai2023sigmoid} for language supervision, ViT-MAE-base~\citep{he2022masked} for self-supervision, and ResNet-50~\citep{he2016deep} for class supervision. The pool is balanced (two encoders per paradigm), and every individual encoder is competent at image classification. Paradigms (3) generative and (5) task-expert do not appear here. They are reintroduced in the main experiments of Section~\ref{sec:experiments}, where they serve as adversarial candidates that a competent selector should reject.

We evaluate this pool on five datasets, one per recognition regime: CIFAR-100 (generic object), Oxford Pets (fine-grained object), DTD (texture), Country211 (geo-spatial), and GTSRB (symbol or digit). The same five datasets reappear unchanged in the main experiments. The fusion operator used to produce $F_S$ is the gated-fusion operator of~\citet{shi2024eagle} with a 2-layer MLP, so peak-then-decline cannot reduce to a feature-concatenation artifact.

Encoders enter in a fixed \emph{literature-tracked path}: CLIP, DINOv2, ConvNeXt, SigLIP, MAE, ResNet-50, the Cambrian-1 trio followed by within-paradigm second representatives, so positions $5$--$6$ act as redundancy probes against $1$--$3$. The path is fixed a priori from published fusion practice, using no held-out task, transferability score, or observation of the evaluation datasets. Its purpose is not to approximate the optimal ordering but to show that even a plausible fusion trajectory has an order-specific peak below the full pool; Section~\ref{sec:method} then asks whether a better, more task-aligned trajectory can be built without downstream retraining.

Classical multi-view learning attributes fusion gain to \emph{complementarity}~\citep{xu2013multi}. Our pool spans three paradigms with two competent encoders each, creating conditions intended to favour complementary gains; peak-then-decline therefore appears even on a path designed to favour the ``more is better'' hypothesis. Figure~\ref{fig:overview} previews the central observation (per-dataset curves in supplementary App.~\ref{supp-app-y:scaling-table}).

\subsection{Evidence across regimes and data scales}
\label{sec:why-decline}

The marginal-utility account predicts (i) peak-then-decline whenever marginal gain $g_k$ diminishes faster than cost $c_k$ accumulates and, under the scaling condition stated at the end of Section~\ref{sec:peak-theory}, (ii) a right-shift of the crossing as $n$ grows. We test (i) with the 10-shot scan in Table~\ref{tab:fulldata} and (ii) with the $n$-sweep in Table~\ref{tab:nscaling}, running the same five datasets at 10-shot, 100-shot, and full training set.

\begin{table}[t]
\centering
\caption{\textbf{Peak-then-decline at 10-shot across all five recognition regimes.} Encoders are added one at a time along the literature-tracked path of Section~\ref{sec:diag-setup}. Each cell is gated-fusion accuracy (\%) after the indicated set is in place. \textbf{Bold}: per-column peak.}
\label{tab:fulldata}
\originalpaperlabel{main-tab:fulldata}
\small
\setlength{\tabcolsep}{5pt}
\begin{tabular}{l ccccc}
\toprule
Regime           & generic    & fine-grained & texture & geo-spatial & symbol \\
Dataset          & CIFAR-100  & Pets         & DTD     & Country211  & GTSRB  \\
\midrule
1 (CLIP)         & 60.29 & 85.04 & 64.41 & $\mathbf{13.50}$ & $\mathbf{61.15}$ \\
2 (+DINOv2)      & 80.11 & 92.86 & 72.67 & 9.85             & 51.54 \\
3 (+ConvNeXt)    & $\mathbf{80.37}$ & $\mathbf{93.75}$ & 73.52 & 9.91 & 52.73 \\
4 (+SigLIP)      & 80.18 & 93.46 & 74.05 & 9.82             & 52.34 \\
5 (+MAE)         & 80.20 & 93.65 & 74.17 & 9.71             & 51.97 \\
6 (+ResNet-50)   & 79.92 & 93.49 & $\mathbf{74.18}$ & 9.58  & 52.58 \\
\midrule
Peak position    & 3     & 3     & 6     & 1                & 1     \\
\bottomrule
\end{tabular}
\end{table}

\begin{table}[t]
\centering
\caption{\textbf{$n$-scaling across all five recognition regimes: peak position along the literature-tracked path at three sample sizes.} For each dataset and each of three sample sizes (10 shots per class, 100 shots per class, full training set), we report the peak position along the path and the gated-fusion accuracy (\%) at that peak. The right-most column gives the shift $\Delta = $ (peak at full data) $-$ (peak at 10-shot). The peak moves right with $n$ on $3$ of $5$ datasets (CIFAR-100, Pets, GTSRB); Country211 is single-encoder saturated and stays at $k = 1$; DTD shows a non-monotonic dependence on $n$ (further discussed in the text). Full per-$k$ curves are in supplementary App.~\ref{supp-app-y:scaling-table}.}
\label{tab:nscaling}
\small
\setlength{\tabcolsep}{5pt}
\begin{tabular}{l cc cc cc c}
\toprule
 & \multicolumn{2}{c}{$n = 10$-shot} & \multicolumn{2}{c}{$n = 100$-shot} & \multicolumn{2}{c}{$n = $ full data} & \\
\cmidrule(lr){2-3} \cmidrule(lr){4-5} \cmidrule(lr){6-7}
Dataset (regime)             & peak\,$k$ & Acc & peak\,$k$ & Acc & peak\,$k$ & Acc & $\Delta$ \\
\midrule
CIFAR-100 (generic)          & 3 & 80.37 & 3 & 84.26 & 5 & 87.36 & +2 \\
Pets (fine-grained)          & 3 & 93.75 & 4 & 95.86 & 5 & 95.87 & +2 \\
DTD (texture)                & 6 & 74.18 & 3 & 80.78 & 5 & 82.53 & -1 \\
Country211 (geo-spatial)     & 1 & 13.50 & 1 & 24.01 & 1 & 25.65 & 0  \\
GTSRB (symbol / digit)       & 1 & 61.15 & 1 & 82.35 & 5 & 86.55 & +4 \\
\bottomrule
\end{tabular}
\end{table}

\textbf{Order-specific peaks on a favourable fixed path.} Under 10-shot training (Table~\ref{tab:fulldata}) the peak position is task-dependent: $k = 3$ on CIFAR-100 and Pets, $k = 6$ on DTD, and $k = 1$ on Country211 and GTSRB, where CLIP alone already beats every multi-encoder prefix and each added view has negative net utility along this path. The peak is a property of the ordered marginal-utility sequence, not the model count.

\textbf{Peak shifts right where finite-sample cost dominates.} The $n$-sweep (Table~\ref{tab:nscaling}) shows the peak moving right with $n$ on three of five datasets: GTSRB jumps from $k = 1$ to $k = 5$ at full data, CIFAR-100 and Pets shift by $+2$, Country211 stays at $k = 1$, and DTD is non-monotone. The lesson is structural: the prefix-curve height at \emph{every} $k$ depends on the ordering, so a selector that front-loads high-utility views raises the curve everywhere, which KAGES (Section~\ref{sec:method}) implements.

\section{KAGES: Greedy View-Set Composition by Kernel Alignment}
\label{sec:method}

Section~\ref{sec:phenomenon} showed that large-pool fusion is an ordering problem: a selector's quality is reflected in the entire prefix-accuracy curve, not in a single best prefix. KAGES constructs such an ordering by greedily maximising a set-level kernel-alignment utility. It is training-free but not label-free: it uses labels only through a label kernel and never trains a downstream classifier for a candidate subset, which makes it deployable at pool sizes where validation-based subset search is infeasible.

\subsection{Set-level objective and greedy selection}
\label{sec:obj}

Let $\Ms = \{m_1, \ldots, m_M\}$ be the candidate pool, $F_s \in \mathbb{R}^{n \times d_s}$ the feature matrix of encoder $s$ on $n$ labelled samples, and $Y \in \{0,1\}^{n \times C_{\mathrm{cls}}}$ the one-hot label matrix. For a subset $S \subseteq \Ms$, KAGES assigns the set-level score
\begin{equation}
\label{eq:obj}
J(S) \;=\; U(\bar F_S, Y) \;-\; \tau \sum_{s \in S} \frac{d_s}{n},
\end{equation}
where $U(\bar F_S, Y) \in \mathbb{R}$ measures the task alignment of the joint selected representation $\bar F_S$ and the optional second term is a per-encoder finite-sample complexity penalty ($\tau \ge 0$; the default $\tau = 0$ recovers the pure utility and is swept in Section~\ref{sec:experiments}). We set $U(\emptyset,Y):=J(\emptyset):=0$ by convention. Given a current set $S$, KAGES scores a candidate $m \notin S$ by the marginal gain $\Delta J(m \mid S) := J(S \cup \{m\}) - J(S)$, adds the candidate with the largest marginal gain, and repeats until all candidates are ordered. The output is a \emph{complete} encoder ordering; downstream users may select any prefix according to their compute budget. When a single score-based deployment size is desired, an optional rule is $\hat k_J := \min\!\left(\operatorname*{arg\,max}_{1\le k\le M} J(S_k)\right)$, which chooses the smallest prefix attaining the largest score along the completed path.

\begin{algorithm}[h]
\caption{KAGES: training-free greedy encoder ordering by set-level marginal utility.}
\label{alg:kages}
\originalpaperlabel{main-alg:kages}
\begin{algorithmic}[1]
\REQUIRE Pool $\Ms = \{m_1, \ldots, m_M\}$; labelled task $(X, Y)$ with $n$ samples; cost coefficient $\tau \ge 0$.
\STATE Compute centred encoder kernels $K_s^c$ for every $s \in \Ms$ on the $n$ samples.
\STATE Initialize $S_0 \leftarrow \emptyset$, $\hpi \leftarrow (\,)$.
\FOR{$t = 1, 2, \ldots, M$}
  \STATE $m_t \leftarrow \arg\max_{m \notin S_{t-1}}\; \Delta J(m \mid S_{t-1})$ \hfill {\itshape\small\# marginal-gain greedy step}
  \STATE $S_t \leftarrow S_{t-1} \cup \{m_t\}$; append $m_t$ to $\hpi$.
\ENDFOR
\STATE Optional: $\hat k_J \leftarrow \min\!\left(\operatorname*{arg\,max}_{1\le k\le M} J(S_k)\right)$.
\RETURN encoder ordering $\hpi = (m_1, \ldots, m_M)$ and optional score-selected prefix $S_{\hat k_J}$.
\end{algorithmic}
\end{algorithm}

Algorithm~\ref{alg:kages} produces a complete ordering of the pool and, optionally, a score-selected prefix. The algorithm itself does not require monotonicity or submodularity; Corollary~\ref{cor:surrogate-peak} applies only to the unpenalised alignment utility and only under its stated conditions.

\subsection{Centred kernel-target alignment as utility}
\label{sec:ker-evid}

For non-empty $S$, we instantiate $U$ with the centred set-level kernel-target alignment,
\begin{equation}
\label{eq:U-ker}
U(\bar F_S, Y) \;=\; \mathrm{CKA}\!\bigl(K_S^c,\, K_Y^c\bigr),
\qquad
K_S^c \;=\; \sum_{s \in S} K_s^c,
\quad
K_s^c \;=\; H\, F_s F_s^{\top}\, H,
\quad
K_Y^c \;=\; H\, Y Y^{\top}\, H,
\end{equation}
where $H = I_n - \tfrac{1}{n}\mathbf{1}\mathbf{1}^{\top}$ is the centering matrix and $K_s^c$ is the centred linear kernel of encoder $s$. By the convention above, $U(\emptyset,Y)=0$; we assume $\|K_Y^c\|_F>0$ and discard any encoder with $\|K_s^c\|_F=0$. For non-zero kernels, centred kernel-target alignment~\citep{cristianini2002kernel,kornblith2019similarity} is bounded in $[0, 1]$, scale-invariant in each kernel individually, and reduces to canonical kernel-target alignment for $|S| = 1$.

\emph{Linear-kernel additivity.} The additive form $K_S^c = \sum_{s \in S} K_s^c$ is not cosmetic: dropping the centering shows it equals the linear kernel of the concatenated feature matrix,
\begin{equation}
\label{eq:additivity}
[F_{s_1} \mid \cdots \mid F_{s_k}]\,[F_{s_1} \mid \cdots \mid F_{s_k}]^{\top}
\;=\; \sum_{j=1}^{k} F_{s_j} F_{s_j}^{\top},
\end{equation}
so adding an encoder is an $n \times n$ additive update on the running set kernel, independent of $d_m$. A per-encoder Frobenius-normalised variant $\widetilde{K}_s^c := K_s^c / (\|K_s^c\|_F + \varepsilon)$ yields slightly lower empirical AULC on our pool; for non-zero kernels the $\varepsilon=0$ form is exactly scale-invariant, while the small numerical $\varepsilon>0$ only stabilises near-zero kernels. We report this variant as an ablation in supplementary App.~\ref{supp-app-y:normalization-ablation} and use the raw additive variant in the main results.

\emph{Running scalars and cost.} We maintain $K_S^c$ and three Frobenius scalars ($A_S = \langle K_S^c, K_Y^c \rangle_F$, $B_S = \langle K_S^c, K_S^c \rangle_F$, $C = \|K_Y^c\|_F^2$); each candidate evaluation costs three $n \times n$ Frobenius inner products and one $n \times n$ tensor sum --- $O(n^2)$ per candidate per step, independent of encoder dimension. For 10 encoders at $n = 500$, one full greedy traversal completes in under one second on a single GPU. Writing $p_c=n_c/n$ for the empirical class proportions, the centred label kernel satisfies $(K_Y^c)_{ij}=\mathbf{1}[y_i=y_j]-p_{y_i}-p_{y_j}+\sum_c p_c^2$, reducing to $\mathbf{1}[y_i=y_j]-1/C_{\mathrm{cls}}$ for class-balanced samples. Thus $\langle K_m^c, K_Y^c \rangle_F$ rewards encoder similarity structure that covaries with the task labels.

\subsection{Algebraic structure and greedy guarantee}
\label{sec:prop}

The set-level kernel-target alignment admits an exact algebraic expansion that exposes the relevance--redundancy trade-off without any linear combination of separately estimated proxies. Unless stated otherwise, this subsection analyses the pure alignment utility $U$, equivalently $J$ with $\tau=0$, which is used in the main results. For $\tau>0$, $\Delta J(m\mid S)=\Delta U(m\mid S)-\tau d_m/n$.

\begin{proposition}[Algebraic expansion of $\mathrm{CKA}(K_S^c, K_Y^c)$]
\label{prop:expansion}
For any non-empty $S \subseteq \Ms$ with $\|K_S^c\|_F > 0$ and $\|K_Y^c\|_F > 0$,
\begin{equation}
\label{eq:expansion}
\mathrm{CKA}\!\bigl(K_S^c,\, K_Y^c\bigr) \;=\;
\frac{\displaystyle \sum_{s \in S} \langle K_s^c, K_Y^c \rangle_F}%
     {\sqrt{\displaystyle \sum_{s, s' \in S} \langle K_s^c, K_{s'}^c \rangle_F}\;\cdot\;\|K_Y^c\|_F}.
\end{equation}
The numerator is the sum of per-encoder centred kernel-target inner products. The denominator decomposes into per-encoder self-terms $\langle K_s^c, K_s^c \rangle_F$ (encoder ``kernel mass'') and pairwise cross-terms $\langle K_s^c, K_{s'}^c \rangle_F$ for $s \neq s'$ (between-encoder kernel similarity).
\end{proposition}

\begin{proof}[Proof sketch]
By construction $K_S^c = \sum_{s \in S} K_s^c$ (Eq.~\ref{eq:U-ker}). Bilinearity of the Frobenius inner product yields $\langle K_S^c, K_Y^c \rangle_F = \sum_{s \in S} \langle K_s^c, K_Y^c \rangle_F$ and $\langle K_S^c, K_S^c \rangle_F = \sum_{s, s' \in S} \langle K_s^c, K_{s'}^c \rangle_F$. Substituting into the definition of centred kernel-target alignment recovers Eq.~\eqref{eq:expansion}.
\end{proof}

\paragraph{Exact marginal-gain criterion.}
For brevity write the five scalars
\(
A_S := \langle K_S^c, K_Y^c \rangle_F,\;
B_S := \langle K_S^c, K_S^c \rangle_F,\;
r_m := \langle K_m^c, K_Y^c \rangle_F,\;
h_m := \langle K_m^c, K_S^c \rangle_F,\;
q_m := \langle K_m^c, K_m^c \rangle_F,
\)
with $\|K_Y^c\|_F^2 = C$ fixed across candidates.

\begin{proposition}[Exact CKA marginal-gain criterion]
\label{prop:exact-margin}
For the unpenalised objective used in the main results ($\tau=0$), let $S$ be non-empty with $B_S>0$ and $C>0$. Adding $m \notin S$ to $S$ yields positive alignment gain $\Delta U(m \mid S) > 0$ if and only if
\begin{equation}
\label{eq:exact-margin}
r_m \;>\; A_S \cdot \left( \sqrt{1 + \frac{2 h_m + q_m}{B_S}} \;-\; 1 \right).
\end{equation}
\end{proposition}

\emph{Proof.} Substitute $K_{S \cup \{m\}}^c = K_S^c + K_m^c$ into $U(S \cup \{m\}) = (A_S + r_m) / \sqrt{(B_S + 2 h_m + q_m)\,C}$ and require $U(S \cup \{m\}) > U(S) = A_S / \sqrt{B_S\,C}$. With $A_S \ge 0$ both sides are non-negative; squaring and rearranging yields Eq.~\eqref{eq:exact-margin}. Full derivation in supplementary App.~\ref{supp-app-y:exact-margin}. \hfill$\square$

\paragraph{Adaptive kernel mRMR interpretation.}
For $A_S>0$, when the candidate's target and set-mass contributions are both small relative to the current set, $r_m/A_S \ll 1$ and $(2 h_m + q_m) / B_S \ll 1$, first-order expansion of Eq.~\eqref{eq:exact-margin}, with their product treated as second order, gives the linearised criterion
\begin{equation}
\label{eq:mrmr-linear}
\Delta U(m \mid S) \;\approx\; \frac{1}{\sqrt{B_S\,C}}\,\Bigl[\, r_m \;-\; \tfrac{A_S}{2 B_S}\,\bigl(\,2 h_m + q_m\,\bigr) \Bigr].
\end{equation}
KAGES therefore locally implements a parameter-free relevance-minus-redundancy criterion whose redundancy weight $A_S / (2 B_S)$ is determined by the current set's alignment and mass, rather than being fixed ahead of time as in classical mRMR-style selectors~\citep{peng2005feature,brown2012conditional}.

\paragraph{Conditional greedy guarantee.} The unnormalised numerator $A_S=\langle K_S^c,K_Y^c\rangle_F$ is monotone because the centred Gram matrices are positive semidefinite, but the normalised CKA utility $U(S)$ need not be monotone: its denominator can outweigh a non-negative numerator increment, as characterised exactly by Proposition~\ref{prop:exact-margin}. We therefore state a conditional guarantee under monotonicity and the standard submodularity-ratio condition of \citet{das2011submodular}; the theorem does not assert that every empirical CKA instance satisfies these assumptions.

\begin{theorem}[Conditional greedy guarantee under a submodularity ratio]
\label{thm:y:weak-submod}
Let $U:2^{\Ms}\to\mathbb{R}_{\ge0}$ satisfy $U(\emptyset)=0$ and be monotone non-decreasing. Fix a cardinality budget $K$. Suppose there is a constant $\gamma\in(0,1]$ such that, for every $L\subseteq\Ms$ and every $T\subseteq\Ms\setminus L$ with $|T|\le K$,
\[
\sum_{m\in T}\Delta U(m\mid L)
\;\ge\;
\gamma\bigl[U(L\cup T)-U(L)\bigr].
\]
Then the cardinality-$K$ prefix $S_K^{\mathrm{greedy}}$ produced by Algorithm~\ref{alg:kages} with $\tau=0$ satisfies
\[
U(S_K^{\mathrm{greedy}})
\;\ge\;
\bigl(1-e^{-\gamma}\bigr)\max_{|S|\le K}U(S).
\]
\end{theorem}

\emph{Proof sketch.} Let $O$ be an optimal set with $|O|\le K$. The submodularity-ratio condition implies that the sum of the available gains from $O\setminus S_k$ is at least $\gamma[U(O)-U(S_k)]$. Hence the greedy gain is at least $\gamma/K$ times the current optimality gap, which contracts geometrically and yields the stated $(1-e^{-\gamma})$ factor. Full proof in supplementary App.~\ref{supp-app-y:proof-greedy}.

\begin{corollary}[Prefix-wise fixed-budget guarantee]
\label{cor:surrogate-peak}
Under the hypotheses of Theorem~\ref{thm:y:weak-submod}, for every $k=1,\ldots,K$, the greedy ordering $\pi_g$ produced by Algorithm~\ref{alg:kages} with $\tau=0$ satisfies
\[
U(S_k^{\pi_g})
\;\ge\;
\bigl(1-e^{-\gamma}\bigr)\max_{|S|\le k}U(S).
\]
\end{corollary}

Corollary~\ref{cor:surrogate-peak} is a fixed-budget statement at each prefix size, not a guarantee that empirical CKA is monotone or that its score-selected stopping point is accuracy-optimal. The accuracy-level prefix quality is evaluated empirically in Section~\ref{sec:experiments}.

\paragraph{Empirical ordering consistency.}
KAGES uses empirical marginal gains $\Delta\widehat J_n(m\mid S)$ rather than the marginal gains of a deterministic population target utility $J$. Under the uniform subgaussian concentration and bias assumptions stated in supplementary App.~\ref{supp-app-y:T1}, with scale $\sigma$, bias constant $B$, and population-path margin $G>0$, if $n\ge(4B/G)^2$ then the empirical ordering matches the population greedy ordering with probability at least $1-[M(M+1)-2]\exp(-nG^2/(32\sigma^2))$. This is an assumption-conditional concentration result, not an unconditional Hoeffding claim for empirical CKA.

\section{Experiments}
\label{sec:experiments}

We evaluate KAGES on the main image-classification protocol of Section~\ref{sec:diag-setup} extended to all five pre-training paradigms, and then test transfer to image retrieval and frozen-LLM fusion.

\subsection{Experimental setup}
\label{sec:setup-exp}

\paragraph{Datasets.} Following the taxonomies introduced in Section~\ref{sec:diag-setup}, we evaluate on the same five datasets used by the diagnostic, covering all five recognition regimes one-to-one: \emph{Generic Object} (CIFAR-100), \emph{Fine-grained Object} (Oxford Pets), \emph{Texture / Material} (DTD), \emph{Geo-spatial / Scene} (Country211), and \emph{Structured Symbol / Digit} (GTSRB).

\paragraph{Encoder pool.} The pool extends the diagnostic pool to all five pre-training paradigms, contributing 10 frozen backbones spanning: language-supervised (CLIP-ViT-B/16~\citep{radford2021learning}, SigLIP-B/16~\citep{zhai2023sigmoid}); self-supervised (DINOv2-base~\citep{oquab2023dinov2}, ViT-MAE-base~\citep{he2022masked}, Data2Vec~\citep{baevski2022data2vec}); generative (Stable-Diffusion VAE~\citep{rombach2022sdvae}); class-supervised (ConvNeXt-base~\citep{liu2022convnet}, ResNet-50~\citep{he2016deep}, ViT-B/16~\citep{dosovitskiy2021vit}); task-expert (SAM-ViT-B~\citep{kirillov2023sam}). Paradigms C and E reappear as \emph{adversarial cases} that a competent selector should reject.

\paragraph{Protocols.} We report three headline protocols. \textbf{B} uses 16-shot training with a Concat+MLP fusion head (hidden dim 512, 20 epochs). \textbf{H} uses an 8-shot setting with a larger $M=15$ candidate pool. \textbf{F} uses the full training set with the same fusion-head family. Together they test whether the ordering-quality advantage persists across low-shot, larger-pool, and full-data regimes. Additional protocol and training details are in supplementary App.~\ref{supp-app-y:config}.

\paragraph{Metrics.} View-set composition is an ordering problem, so we evaluate selectors by the quality of their entire prefix path. Our two primary ordering metrics are:
\begin{itemize}[nosep,leftmargin=*]
  \item \textbf{AULC} (normalised area under the $k$--accuracy curve, analogous to AUROC): the mean prefix accuracy along the method's full ordering.
  \item \textbf{Acc@$k^\star$}: the peak accuracy along the ordering, the best fused set the path reaches (marked by stars in Figure~\ref{fig:path-curves}).
\end{itemize}
AULC does not depend on a stopping rule, so it directly evaluates prefix-path quality. Acc@$k^\star$ is a hindsight path diagnostic computed from the downstream accuracy curve; it does not evaluate the optional score-based stopping rule $\hat k_J$, for which we make no separate empirical stopping-point claim.

\paragraph{Baselines and seeds.} We compare KAGES against multi-view selection strategies operating on the same frozen-encoder pool. The primary baselines are the standard multi-view defaults: \textbf{Fuse-All} embodies the more-is-better default, \textbf{Diversity} selects for representational difference through centred kernel alignment between encoders, and \textbf{Random} is a random-order baseline. Because AULC requires a complete path, Fuse-All is represented by the prespecified fixed-order prefix path used by the implementation, whose final prefix contains the full pool. We additionally compare against two stronger subset-selection baselines, \textbf{DPP} (determinantal point process MAP inference) and \textbf{Submodular} (a facility-location coverage objective). \textbf{Oracle} greedily adds the encoder producing the largest true downstream-accuracy improvement and is a validation-based greedy hindsight reference, not a global upper bound over all subsets. All AULC numbers are means over five random seeds $\{17, 23, 42, 71, 137\}$; standard deviations in supplementary App.~\ref{supp-app-y:tables-full}.

\subsection{Main results}
\label{sec:main-results}

\begin{table}[t]
\centering
\caption{\textbf{Per-dataset AULC (\%) on all three protocols: KAGES beats the multi-view defaults on every dataset, by a wide margin where heterogeneous fusion matters.} Against full fusion (Fuse-All), random and diversity selection, KAGES is best in every cell. The $\Delta$ rows give KAGES minus Fuse-All: the gain is largest on the structured-symbol task GTSRB ($+8.9$ to $+15.5$ points) and on the geo regime at full data. Stronger subset-selection baselines (DPP, submodular) are reported in the text; KAGES matches or exceeds them too. Means over 5 seeds $\{17,23,42,71,137\}$; \textbf{bold}: best per column within a protocol.}
\label{tab:cross-protocol}
\footnotesize
\setlength{\tabcolsep}{5.5pt}
\begin{tabular}{l ccccc c}
\toprule
Method & C100 & Pets & DTD & C211 & GTSRB & Avg \\
\midrule
\multicolumn{7}{l}{\emph{Protocol B (16-shot, Concat+MLP)}} \\
Fuse-All  & 78.54 & 94.06 & 75.35 & 10.16 & 57.18 & 63.06 \\
Random    & 77.06 & 92.97 & 74.87 & 11.10 & 59.16 & 63.03 \\
Diversity & 80.61 & 90.91 & 74.38 & 11.27 & 59.12 & 63.26 \\
\textbf{KAGES (ours)} & \textbf{83.84} & \textbf{94.70} & \textbf{77.82} & \textbf{12.16} & \textbf{66.12} & \textbf{66.93} \\
\textit{\small $\Delta$ vs Fuse-All} & \textit{\small +5.3} & \textit{\small +0.6} & \textit{\small +2.5} & \textit{\small +2.0} & \textit{\small +8.9} & \textit{\small +3.9} \\
\midrule
\multicolumn{7}{l}{\emph{Protocol H (8-shot, $M{=}15$)}} \\
Fuse-All  & 73.89 & 89.01 & 70.10 & 9.65 & 52.42 & 59.02 \\
Random    & 75.11 & 91.78 & 69.84 & 9.77 & 54.88 & 60.28 \\
Diversity & 78.18 & 86.91 & 68.87 & 9.67 & 56.67 & 60.06 \\
\textbf{KAGES (ours)} & \textbf{80.20} & \textbf{92.82} & \textbf{72.36} & \textbf{10.57} & \textbf{67.92} & \textbf{64.77} \\
\textit{\small $\Delta$ vs Fuse-All} & \textit{\small +6.3} & \textit{\small +3.8} & \textit{\small +2.3} & \textit{\small +0.9} & \textit{\small +15.5} & \textit{\small +5.8} \\
\midrule
\multicolumn{7}{l}{\emph{Protocol F (full data)}} \\
Fuse-All  & 85.07 & 95.90 & 80.05 & 16.51 & 79.71 & 71.45 \\
Random    & 84.66 & 95.26 & 80.07 & 19.43 & 81.87 & 72.26 \\
Diversity & 86.21 & 94.35 & 79.85 & 18.51 & 81.31 & 72.05 \\
\textbf{KAGES (ours)} & \textbf{88.19} & \textbf{96.25} & \textbf{82.96} & \textbf{20.07} & \textbf{86.41} & \textbf{74.78} \\
\textit{\small $\Delta$ vs Fuse-All} & \textit{\small +3.1} & \textit{\small +0.4} & \textit{\small +2.9} & \textit{\small +3.6} & \textit{\small +6.7} & \textit{\small +3.3} \\
\bottomrule
\end{tabular}
\end{table}

\begin{figure}[t]
\centering
\includegraphics[width=0.74\textwidth]{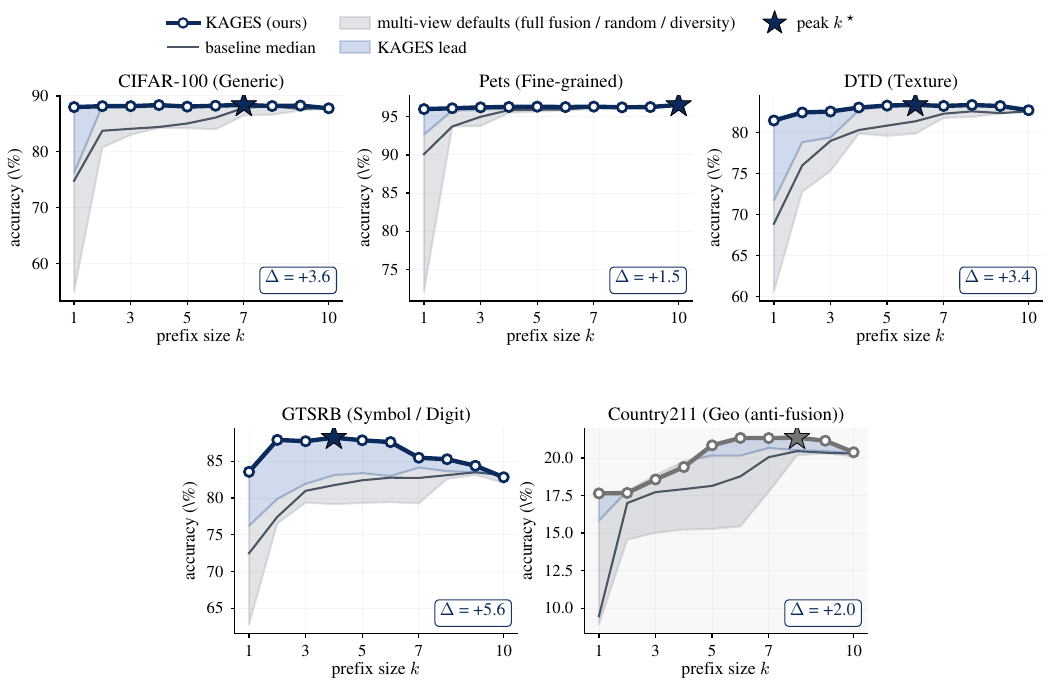}
\caption{\textbf{KAGES front-loads high-utility encoders and stays at the top of the multi-view baseline envelope (Protocol F, full data).} Per-dataset prefix accuracy along each method's ordering. KAGES (navy, star at its peak) is plotted over the gray envelope spanning the standard multi-view defaults (full fusion, random, and diversity selection), with their per-$k$ median in gray. KAGES reaches a strong representation at small $k$ on every regime while the defaults climb slowly, with the clearest separation on the texture and symbol tasks. Country211 is the anti-fusion regime in which CLIP alone saturates geolocation, so all methods converge.}
\label{fig:path-curves}
\end{figure}

\textbf{KAGES beats the multi-view defaults on every dataset, and the margin is wide where fusion matters.} Table~\ref{tab:cross-protocol} reports per-dataset AULC on all three protocols. KAGES is best in every cell against full fusion, random and diversity selection. The advantage concentrates in the regimes that genuinely reward heterogeneous fusion: on the structured-symbol task GTSRB it beats blind full fusion by $+8.9$, $+15.5$ and $+6.7$ points on B, H and F, and on the geo regime at full data by $+3.6$. On saturated natural-image tasks such as Pets the ordering still matters but the margin is smaller, since a single strong encoder already approaches the peak. Averaged over the five regimes KAGES leads Fuse-All by $+3.9$, $+5.8$ and $+3.3$ points. The path-curve view (Figure~\ref{fig:path-curves}) shows the mechanism: KAGES front-loads high-utility encoders and is already strong at small $k$, while the defaults climb slowly and recover late, if at all.

\textbf{KAGES also matches or exceeds stronger subset-selection baselines.} Beyond the more-is-better defaults, we compare against two strong subset-selection methods on the same pool: DPP MAP inference and a facility-location submodular objective. Both are far better than full fusion, yet KAGES still leads on average AULC on every protocol (B: $66.93$ vs $66.13$ submodular and $65.21$ DPP; H: $64.77$ vs $63.49$ and $62.13$; F: $74.78$ vs $73.98$ and $73.51$), confirming that the label-aware kernel-alignment criterion adds value over diversity- and coverage-driven selection.

\textbf{First-pick consistently favours self-supervised encoders, and adversarial paradigms are rejected.} On four natural-image regimes (CIFAR-100, Pets, DTD, Country211) KAGES's first pick is DINOv2-base (paradigm B) on all 5 seeds; on Structured Symbol / Digit (GTSRB) it picks the class-supervised ResNet-50 (3/5 seeds) or CLIP (2/5). \emph{Paradigms C (generative) and E (task-expert) are never first-picked}, consistent with both being structurally orthogonal to classification.

\paragraph{Cost-penalty ablation ($\tau$).} Sweeping $\tau \in \{0, 0.01, 0.05, 0.1, 0.5, 1.0\}$ (3 seeds) gives mean AULC $\{66.9, 66.7, 65.5, 64.6, 62.9, 62.5\}$, maximised at $\tau = 0$. Larger $\tau$ makes the ordering track encoder dimension and degrades fast on GTSRB, so the main results use the pure $\mathrm{CKA}(K_S^c, K_Y^c)$ criterion.

\subsection{Beyond classification}
\label{sec:cross-modal}

KAGES's advantage extends to two settings outside the recognition-regime taxonomy, and in the language setting the peak-then-decline reappears. (i)~\emph{Image retrieval}: on the same encoder pool with L2-normalised concatenated features and nearest-neighbour search, Recall@$1$ saturates later than classification, peaking at $k = 8$ on Pets and $k = 10$ on both DTD and GTSRB, with KAGES matching the peak in each case (full curves in supplementary App.~\ref{supp-app-y:retrieval-full}). (ii)~\emph{Frozen-LLM fusion}: treating six base LLMs (Llama-3, Qwen-2.5, Mistral, Gemma, Yi-1.5, DeepSeek) as views and training a linear probe on MMLU and ARC-Challenge, accuracy peaks then declines as in vision; KAGES reaches $75.0\%$ on MMLU at $k = 2$ and $91.0\%$ on ARC at $k = 4$, hitting the peak at a smaller $k$ than fixed-order or diversity-greedy baselines (Figure~\ref{fig:llm-kcurve}; full results in supplementary App.~\ref{supp-app-y:llm}).

\begin{figure}[t]
\centering
\includegraphics[width=0.62\textwidth]{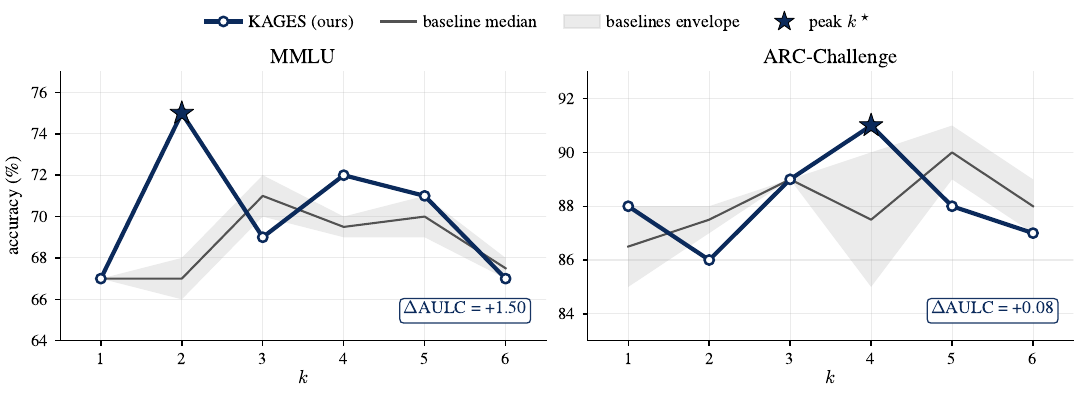}
\caption{\textbf{Accuracy curves over fused LLM count} for MMLU and ARC-Challenge. The gray band spans diversity-greedy and fixed-order baselines, the gray line shows their median, and KAGES is shown in navy with stars marking its peak. The peak-then-decline pattern reproduces at the LLM layer.}
\label{fig:llm-kcurve}
\end{figure}

\section{Related Work}
\label{sec:related}

\paragraph{Foundation-model fusion and transferability scoring.} Classical multi-view learning~\citep{xu2013multi,li2019mvlearn,trosten2023mvclust} fuses a small fixed set of feature views and offers no recipe for selecting views out of a heterogeneous pool. Multi-encoder MLLMs scale this to multiple frozen vision backbones: Eagle~\citep{shi2024eagle}, Cambrian-1~\citep{tong2024cambrian}, BRAVE~\citep{kar2024brave}, and MoVA~\citep{zong2024mova} study encoder combinations or routing within relatively small, predefined expert pools. Training-free transferability scores (LogME~\citep{you2021logme,you2022ranking}, H-Score~\citep{bao2019hscore}, GBC~\citep{pandy2022gbc}, LEEP~\citep{nguyen2020leep}, NCE~\citep{tran2019transferability}, TransRate~\citep{huang2022transrate}, EMMS~\citep{meng2023emms}, PACTran~\citep{ding2022pactran}, OTCE~\citep{tan2021otce,tan2024fotce}, DISCO~\citep{zhang2025disco}) score a single encoder and provide no notion of how a candidate interacts with an already-selected pool. Pool-level approaches are sparse: OSBORN~\citep{vimal2023osborn} gives a $(1-1/e)$ guarantee at a fixed $k=3$; model-zoo studies~\citep{dong2022zood,chen2023zoo} select at the recipe level. Across these lines, encoder choice is either constrained to predefined pools or scored one model at a time, without jointly constructing a complete ordering over a large heterogeneous pool. We depart on three counts that, to our knowledge, no prior work combines: we characterise order-dependent peak-then-decline, including an absolute decline on structured-symbol tasks, rather than only diminishing returns on a handful of encoders; we trace it to the frozen feature-fusion regime through a gain--cost crossing; and we turn that account into a single training-free selector with conditional prefix-wise guarantees that produces a high-quality ordering and provides an optional score-based stopping rule from frozen-feature statistics alone.

\paragraph{Kernel and information-theoretic selection.} Classical scalar feature selection (mRMR~\citep{peng2005feature}, JMI~\citep{yang1999jmi}, Brown et al.~\citep{brown2012conditional}) and similarity tools (CKA~\citep{kornblith2019similarity}, HSIC~\citep{gretton2005measuring}, Platonic Hypothesis~\citep{huh2024platonic}) compare representations or select scalar features but do not produce a path-quality recommendation across prefix sizes. The closest line is greedy kernel-based view selection: HSIC-mRMR / HSIC-Lasso~\citep{song2012feature,yamada2014hsiclasso} and Kernel-JMI~\citep{brown2012conditional} select scalar features via hand-weighted pairwise scores at cost growing with feature dimension. KAGES instead treats each encoder's representation as one atomic unit via the single closed-form pure alignment utility $U(S) = \mathrm{CKA}(K_S^c, K_Y^c)$, with kernel-level candidate evaluation independent of encoder dimension (Section~\ref{sec:ker-evid}); its exact positive-gain criterion (Proposition~\ref{prop:exact-margin}) reproduces relevance-minus-redundancy with adaptive weight $A_S / (2 B_S)$ rather than tuned, and its conditional prefix-wise fixed-budget guarantee (Corollary~\ref{cor:surrogate-peak}) applies at every cardinality for which the stated monotonicity and submodularity-ratio conditions hold.

\section{Discussion, Limitations, and Future Work}
\label{sec:discussion}

\paragraph{Limitations.} The normalised CKA objective is neither generally monotone nor generally submodular because the denominator in Eq.~\ref{eq:expansion} can outweigh a non-negative numerator increment and introduces cross-overlap terms. Our $(1-e^{-\gamma})$ guarantee is therefore conditional on monotonicity and a positive submodularity ratio; establishing when empirical frozen-encoder pools satisfy these conditions is future work. The generative and task-expert paradigms are handled only at the selector level: KAGES never first-picks them, and fusing them remains future work.

\section{Conclusion}
We have recast multi-view learning for the foundation-model era, where a view is one of thousands of frozen encoders and the design problem becomes \emph{view-set composition}: which encoders to fuse, and how many. We showed that the classical complementarity principle fails at this scale: the best achievable accuracy peaks at a small set and then plateaus or, on structured-symbol tasks, declines. We proposed \emph{KAGES}, a training-free selector with encoder-dimension-independent kernel-level candidate evaluation and a conditional $(1{-}e^{-\gamma})$ prefix-wise fixed-budget guarantee under the stated monotonicity and submodularity-ratio conditions. KAGES ranks first among multi-view selection methods across five regimes, five pre-training paradigms, and three evaluation settings, beating full fusion and diversity-based selection while approaching the oracle. As public model hubs keep growing, choosing which encoders to fuse, rather than how to fuse a fixed few, becomes the central question of multi-view learning, and a training-free, theory-backed answer to it is a step toward using the world's frozen models at scale.

\subsubsection*{Acknowledgments and Authorship Note}
We thank Prof. Zongbo Han for his guidance and helpful discussions. Prof. Han declined to be listed as an author of the final manuscript. Consequently, the earlier BMVC 2026 submission was withdrawn after the Area Chair decision and before publication. This preprint lists only the authors who approved the present version.

\bibliographystyle{iclr2027_conference}
\bibliography{references}

\clearpage
\let\label\supppaperlabel
\renewcommand{\thesection}{A.\arabic{section}}
\renewcommand{\thefigure}{A.\arabic{figure}}
\renewcommand{\thetable}{A.\arabic{table}}
\renewcommand{\theequation}{A.\arabic{equation}}
\renewcommand{\thetheorem}{A.\arabic{theorem}}
\renewcommand{\theproposition}{A.\arabic{proposition}}
\renewcommand{\thecorollary}{A.\arabic{corollary}}
\renewcommand{\thelemma}{A.\arabic{lemma}}
\renewcommand{\theassumption}{A.\arabic{assumption}}
\renewcommand{\thedefinition}{A.\arabic{definition}}
\def\theHsection{appendix.\arabic{section}}
\def\theHsubsection{appendix.\arabic{section}.\arabic{subsection}}
\def\theHsubsubsection{appendix.\arabic{section}.\arabic{subsection}.\arabic{subsubsection}}
\def\theHfigure{appendix.\arabic{figure}}
\def\theHtable{appendix.\arabic{table}}
\def\theHequation{appendix.\arabic{equation}}
\def\theHtheorem{appendix.\arabic{theorem}}
\def\theHproposition{appendix.\arabic{proposition}}
\def\theHcorollary{appendix.\arabic{corollary}}
\def\theHlemma{appendix.\arabic{lemma}}
\def\theHassumption{appendix.\arabic{assumption}}
\def\theHdefinition{appendix.\arabic{definition}}
\setcounter{section}{0}\setcounter{figure}{0}\setcounter{table}{0}\setcounter{equation}{0}
\setcounter{theorem}{0}\setcounter{proposition}{0}\setcounter{corollary}{0}
\setcounter{lemma}{0}\setcounter{assumption}{0}\setcounter{definition}{0}

\section*{Appendix Overview}

The appendix supports specific portions of the main paper:

\begin{itemize}[nosep,leftmargin=*]
  \item \textbf{App.~\ref{app-y:proof-greedy}:} conditional $(1-e^{-\gamma})$ greedy approximation guarantee under a standard submodularity-ratio condition, together with a mutual-information specialization that recovers the $(1-1/e)$ bound.
  \item \textbf{App.~\ref{app-y:T1}:} ordering-consistency proof for $\Delta J$ (Theorem~\ref{thm:y:T1}), supporting the empirical-ordering claim at the end of Section~\ref{main-sec:prop} of the main paper.
  \item \textbf{App.~\ref{app-y:exp}:} experimental configuration, full results tables, ordering details, risk decomposition, and proof of the conditional marginal-utility crossing (Theorem~\ref{main-thm:peak}). Supports Sections~\ref{main-sec:phenomenon} and~\ref{main-sec:experiments} of the main paper.
  \item \textbf{App.~\ref{app-y:llm}:} cross-modal validation on frozen LLMs (supports Section~\ref{main-sec:cross-modal}).
  \item \textbf{App.~\ref{app-y:future}:} future work, online and sequential extensions.
\end{itemize}

\noindent The notation of the main paper carries over: $\Ms$ is the encoder pool of size $M$, and $\hpi$ is the KAGES greedy ordering on the labelled task, as defined in Section~\ref{main-sec:method} of the main paper.


\section{Greedy Approximation Guarantee under a Submodularity Ratio}
\label{app-y:proof-greedy}

We give the conditional fixed-budget greedy guarantee used in the main paper. The result is stated for a generic non-negative set utility $U$ and applies only when the stated monotonicity and submodularity-ratio conditions hold. We then give a population mutual-information specialization for which these conditions follow from conditional independence.

For $m \notin S$, write $\Delta U(m \mid S) := U(S \cup \{m\}) - U(S)$.

\begin{theorem}[Conditional greedy guarantee under a submodularity ratio]
\label{thm:y:greedy}
Let $U:2^{\Ms}\to\mathbb{R}_{\ge 0}$ satisfy $U(\emptyset)=0$ and be monotone non-decreasing. Fix a cardinality budget $K$. Suppose there is a constant $\gamma\in(0,1]$ such that, for every $L\subseteq\Ms$ and every $T\subseteq\Ms\setminus L$ with $|T|\le K$,
\begin{equation}
\label{eq:y:submod-ratio}
\sum_{m\in T}\Delta U(m\mid L)
\;\ge\;
\gamma\bigl[U(L\cup T)-U(L)\bigr].
\end{equation}
Let $S_0=\emptyset$ and let greedy choose
\[
m_k\in\arg\max_{m\notin S_{k-1}}\Delta U(m\mid S_{k-1}),
\qquad
S_k=S_{k-1}\cup\{m_k\}.
\]
Then
\[
U(S_K)
\;\ge\;
\bigl(1-e^{-\gamma}\bigr)
\max_{|S|\le K}U(S).
\]
\end{theorem}

\begin{proof}
Let $O\in\arg\max_{|S|\le K}U(S)$ and write $G_k:=U(O)-U(S_k)$. By monotonicity and~\eqref{eq:y:submod-ratio},
\[
\sum_{m\in O\setminus S_k}\Delta U(m\mid S_k)
\;\ge\;
\gamma\bigl[U(S_k\cup O)-U(S_k)\bigr]
\;\ge\;
\gamma G_k.
\]
Since $|O\setminus S_k|\le K$, the largest available marginal gain is at least $\gamma G_k/K$. Greedy therefore satisfies
\[
U(S_{k+1})-U(S_k)\ge \frac{\gamma}{K}G_k,
\qquad
G_{k+1}\le\left(1-\frac{\gamma}{K}\right)G_k.
\]
Iterating from $S_0=\emptyset$ gives
\[
G_K
\le
\left(1-\frac{\gamma}{K}\right)^KU(O)
\le e^{-\gamma}U(O),
\]
which proves the claim.
\end{proof}

\begin{corollary}[Prefix-wise fixed-budget guarantee]
\label{cor:y:prefix}
Under the hypotheses of Theorem~\ref{thm:y:greedy}, for every $k=1,\ldots,K$, the same greedy ordering satisfies
\[
U(S_k)\ge(1-e^{-\gamma})\max_{|S|\le k}U(S).
\]
\end{corollary}

\begin{proof}
Apply Theorem~\ref{thm:y:greedy} with the budget set to $k$. The first $k$ greedy choices are unchanged when the traversal is subsequently continued to budget $K$.
\end{proof}

For the information-theoretic specialization, let
\[
\bar F_S := (F_m)_{m\in S}
\]
denote the joint collection of the selected encoder features.

\begin{lemma}[Conditional independence implies information submodularity]
\label{lem:y:submod}
The set function $f(S):=I(\bar F_S;Y)$ is normalized and monotone non-decreasing. If the encoder features are mutually independent conditional on $Y$, then $f$ is submodular and hence satisfies~\eqref{eq:y:submod-ratio} with $U=f$ and $\gamma=1$.
\end{lemma}

\begin{proof}
Normalization is immediate. Monotonicity follows from the chain rule,
\[
f(A\cup\{m\})-f(A)=I(F_m;Y\mid \bar F_A)\ge0.
\]
For $A\subseteq B$ and $m\notin B$,
\[
I(F_m;Y\mid \bar F_A)
=
H(F_m\mid \bar F_A)-H(F_m\mid Y,\bar F_A).
\]
Conditional independence gives $H(F_m\mid Y,\bar F_A)=H(F_m\mid Y)$, and similarly for $B$. Since conditioning cannot increase entropy,
$H(F_m\mid \bar F_A)\ge H(F_m\mid \bar F_B)$, so
\[
I(F_m;Y\mid \bar F_A)
\ge
I(F_m;Y\mid \bar F_B).
\]
This is the diminishing-returns condition for submodularity. A normalized monotone submodular function satisfies~\eqref{eq:y:submod-ratio} with $\gamma=1$.
\end{proof}

\begin{corollary}[$(1-1/e)$ information-gain guarantee]
\label{cor:y:greedy-mi}
Under the conditional-independence condition of Lemma~\ref{lem:y:submod}, greedy maximization of $f(S)=I(\bar F_S;Y)$ satisfies
\[
f(S_K^{\mathrm{greedy}})
\ge
(1-1/e)\max_{|S|\le K}f(S).
\]
\end{corollary}

\paragraph{Scope of the guarantee.}
Theorem~\ref{thm:y:greedy} is conditional: it does not assert that every frozen-encoder pool, or the empirical CKA objective used by KAGES, is automatically monotone or has a positive submodularity ratio. Because CKA contains a normalization denominator, non-negative target alignment of a candidate kernel alone does not guarantee a non-negative marginal CKA gain; the exact one-step condition is derived in App.~\ref{app-y:exact-margin}. The theorem and Corollary~\ref{cor:y:prefix} provide fixed-budget prefix guarantees only when the stated conditions hold.


\section{Ordering Consistency for $\Delta J$}
\label{app-y:T1}

KAGES (Algorithm~\ref{main-alg:kages} in the main paper) acts on the empirical utility $\widehat J_n$ computed from $n$ labelled samples. This appendix shows that, under uniform subgaussian concentration of the marginal gains and a non-vanishing per-step population margin, empirical KAGES recovers the population greedy ordering with exponentially high probability.

\subsection{Setup and notation}

Let $\Ms = \{1, \ldots, M\}$ be a fixed encoder pool with $M\ge2$. Write $J(S)$ for a deterministic population target utility (for example, a population limit of the alignment score when such a limit exists) and $\widehat J_n(S)$ for its $n$-sample empirical estimator. The theorem below is generic and relies only on the stated concentration and bias assumptions; it does not derive those assumptions for empirical CKA. Marginal gains are
\[
\Delta J(m \mid S) := J(S \cup \{m\}) - J(S),
\qquad
\Delta\widehat J_n(m \mid S) := \widehat J_n(S \cup \{m\}) - \widehat J_n(S).
\]
The population greedy trajectory is $(m_t^\star,S_t^\star)_{t=1}^{M}$ with $S_0^\star=\emptyset$ and
\[
m_t^\star\in\arg\max_{m\notin S_{t-1}^\star}\Delta J(m\mid S_{t-1}^\star),
\qquad
S_t^\star=S_{t-1}^\star\cup\{m_t^\star\}.
\]
The empirical trajectory $(\hat m_t,\hat S_t)$ is defined analogously using $\Delta\widehat J_n$. Ties are resolved by the same fixed deterministic rule in both trajectories.

\subsection{Assumptions}

\begin{assumption}[Sample independence and pool finiteness]
\label{ass:y:T1iid}
The $n$ labelled observations are i.i.d.\ from a fixed joint distribution, and the pool cardinality $M$ is fixed independently of $n$.
\end{assumption}

\begin{assumption}[Uniform subgaussian concentration with bias]
\label{ass:y:T1subg}
For every $m\in\Ms$ and every $S\subseteq\Ms\setminus\{m\}$,
\[
\Pr\!\left[
\left|\Delta\widehat J_n(m\mid S)-\mathbb E\,\Delta\widehat J_n(m\mid S)\right|>t
\right]
\le
2\exp\!\left(-\frac{nt^2}{2\sigma^2}\right)
\quad\text{for all }t>0,
\]
with $\sigma>0$. Moreover, the bias
\[
b_n(m,S):=\mathbb E\,\Delta\widehat J_n(m\mid S)-\Delta J(m\mid S)
\]
satisfies $|b_n(m,S)|\le B/\sqrt n$ uniformly over $(m,S)$.
\end{assumption}

\begin{assumption}[Greedy ordering margin]
\label{ass:y:T1order}
For every step $t\in\{1,\ldots,M-1\}$,
\[
\Delta J(m_t^\star\mid S_{t-1}^\star)
-
\max_{m\notin S_{t-1}^\star\cup\{m_t^\star\}}
\Delta J(m\mid S_{t-1}^\star)
\ge G
\]
for some $G>0$.
\end{assumption}

\subsection{Theorem statement}

\begin{theorem}[Ordering consistency for $\Delta J$]
\label{thm:y:T1}
Under Assumptions~\ref{ass:y:T1iid}--\ref{ass:y:T1order}, if $n\ge(4B/G)^2$, then
\[
\Pr\bigl[\hpi\neq(m_1^\star,\ldots,m_M^\star)\bigr]
\le
\bigl[M(M+1)-2\bigr]
\exp\!\left(-\frac{nG^2}{32\sigma^2}\right).
\]
Consequently, a target failure probability $\epsilon\in(0,1)$ is attained whenever
\[
n
\ge
\max\!\left\{
\left(\frac{4B}{G}\right)^2,
\frac{32\sigma^2}{G^2}
\log\!\frac{M(M+1)-2}{\epsilon}
\right\}.
\]
In particular, the sufficient sample size is
$O\!\bigl((B^2+\sigma^2[\log M+\log(1/\epsilon)])/G^2\bigr)$.
\end{theorem}

\subsection{Proof}

\paragraph{Step 1: deviation from the population marginal gain.}
If $n\ge(4B/G)^2$, then $|b_n(m,S)|\le G/4$. Hence, for any fixed $(m,S)$,
\begin{align*}
&\Pr\!\left[
\left|\Delta\widehat J_n(m\mid S)-\Delta J(m\mid S)\right|\ge\frac G2
\right]\\
&\qquad\le
\Pr\!\left[
\left|\Delta\widehat J_n(m\mid S)-\mathbb E\,\Delta\widehat J_n(m\mid S)\right|\ge\frac G4
\right]
\le
2\exp\!\left(-\frac{nG^2}{32\sigma^2}\right).
\end{align*}

\paragraph{Step 2: deterministic good events along the population path.}
For every population-path step $t\in\{1,\ldots,M-1\}$ and every candidate $m\notin S_{t-1}^\star$, define
\[
\mathcal E_{t,m}
:=
\left\{
\left|\Delta\widehat J_n(m\mid S_{t-1}^\star)-\Delta J(m\mid S_{t-1}^\star)\right|<\frac G2
\right\},
\]
and let $\mathcal E:=\bigcap_{t=1}^{M-1}\bigcap_{m\notin S_{t-1}^\star}\mathcal E_{t,m}$. These events are defined on the deterministic population greedy path before the empirical path is run, so no conditioning on a data-dependent previous-selection event is needed.

On $\mathcal E$, the empirical and population paths agree by induction. Suppose $\hat S_{t-1}=S_{t-1}^\star$. For any remaining competitor $m\neq m_t^\star$, the population margin and the two strict $G/2$ error bounds give
\begin{align*}
&\Delta\widehat J_n(m_t^\star\mid S_{t-1}^\star)
-
\Delta\widehat J_n(m\mid S_{t-1}^\star)\\
&\quad>
\Delta J(m_t^\star\mid S_{t-1}^\star)
-
\Delta J(m\mid S_{t-1}^\star)
-G
\ge0.
\end{align*}
Thus the empirical rule selects $m_t^\star$. Starting from $S_0^\star=\hat S_0=\emptyset$, induction proves equality of the paths through step $M-1$; the final step has only one candidate.

\paragraph{Step 3: one unconditional union bound.}
Step~1 and a union bound over the deterministic family of events yield
\begin{align*}
\Pr(\mathcal E^c)
&\le
2\sum_{t=1}^{M-1}(M-t+1)
\exp\!\left(-\frac{nG^2}{32\sigma^2}\right)\\
&=
\bigl[M(M+1)-2\bigr]
\exp\!\left(-\frac{nG^2}{32\sigma^2}\right).
\end{align*}
Since $\mathcal E$ implies recovery of the full ordering, this is the claimed failure-probability bound. Solving it for a target failure probability $\epsilon$ proves the sample-complexity statement.
\hfill$\blacksquare$

\paragraph{Remark.}
The polynomial prefactor is a conservative consequence of a uniform union bound along the population greedy path. The exponential dependence on $nG^2/\sigma^2$ and the logarithmic dependence on $M$ are the substantive parts of the result.


\section{Experimental Details, Theorems, and Full Tables}
\label{app-y:exp}

This appendix expands the experimental sections of the main paper (Sections~\ref{main-sec:phenomenon}, \ref{main-sec:experiments}). It contains the configuration, the formal risk decomposition, the proof of the conditional marginal-utility crossing (Theorem~\ref{main-thm:peak}), and the full results tables underlying the main-text figures.

\subsection{Experimental configuration}
\label{app-y:config}

\paragraph{Taxonomies (full).} The main paper introduces two taxonomies that recur throughout the experiments. \textbf{Encoder paradigm} partitions pre-trained vision backbones into five categories: (A) \emph{language-supervised} image--text contrastive models (CLIP~\citep{radford2021learning}, SigLIP~\citep{zhai2023sigmoid}); (B) \emph{self-supervised} models trained without labels (DINOv2~\citep{oquab2023dinov2}, MAE~\citep{he2022masked}, Data2Vec~\citep{baevski2022data2vec}); (C) \emph{generative / image-prior} encoders (Stable-Diffusion VAE~\citep{rombach2022sdvae}); (D) \emph{class-supervised} ImageNet-trained models (ConvNeXt~\citep{liu2022convnet}, ResNet-50~\citep{he2016deep}, ViT-B/16~\citep{dosovitskiy2021vit}); (E) \emph{task-expert} encoders trained for dense vision tasks (SAM~\citep{kirillov2023sam}). \textbf{Task recognition regime} partitions the five evaluation datasets one-to-one: \emph{Generic Object} (CIFAR-100), \emph{Fine-grained Object} (Pets), \emph{Texture / Material} (DTD), \emph{Geo-spatial / Scene} (Country211), and \emph{Structured Symbol / Digit} (GTSRB).

\paragraph{Reference-path justification.} The diagnostic reference path CLIP $\to$ DINOv2 $\to$ ConvNeXt $\to$ SigLIP $\to$ MAE $\to$ ResNet-50 is fixed a priori from published fusion practice and not optimised on the evaluation data. Its first three positions are the Cambrian-1~\citep{tong2024cambrian} trio (CLIP, DINOv2, ConvNeXt), reported there as a strong vision-centric combination~\citep{shi2024eagle,kar2024brave,zong2024mova}; the last three add a within-paradigm second representative each (SigLIP for language supervision, MAE for self-supervision, ResNet-50 for class supervision), so positions $5$--$6$ act as redundancy probes against positions $1$--$3$. Self-supervised purely-visual features (DINOv2) are added early to recover what contrastive pre-training misses~\citep{liu2023fromclip}.

\paragraph{Hardware.} All experiments are conducted on a single NVIDIA RTX 5090 GPU with 32GB VRAM.

\paragraph{Model pool.} We use 10 frozen pre-trained encoders spanning all five pre-training paradigms (consistent with Section~\ref{main-sec:experiments} of the main paper): CLIP-ViT-B/16~\citep{radford2021learning} and SigLIP-B/16~\citep{zhai2023sigmoid} (language-supervised); DINOv2-base~\citep{oquab2023dinov2}, ViT-MAE-base~\citep{he2022masked} and Data2Vec~\citep{baevski2022data2vec} (self-supervised); the Stable-Diffusion VAE encoder~\citep{rombach2022sdvae} (generative); ConvNeXt-base~\citep{liu2022convnet}, ResNet-50~\citep{he2016deep} and ViT-B/16~\citep{dosovitskiy2021vit} (class-supervised); and SAM-ViT-B~\citep{kirillov2023sam} (task-expert). The generative and task-expert encoders are reintroduced here as adversarial candidates that a competent selector should reject. Feature dimensions range from 256 (SD-VAE) to 2048 (ResNet-50).


\paragraph{Datasets.} Both the diagnostic protocol (Section~\ref{main-sec:phenomenon}, 6-encoder pool) and the main-results protocol (Section~\ref{main-sec:experiments}, 10-encoder pool) evaluate on the same five datasets, one per recognition regime: CIFAR-100 (100 classes, generic object), Oxford Pets (37 classes, fine-grained object), DTD (47 classes, texture), Country211 (211 classes, geo-spatial scene) and GTSRB (43 classes, structured symbol/digit). Oxford Flowers-102 additionally appears only in the fair 16-shot comparison against single-model CLIP adaptation (Table~\ref{tab:y:vs-single}), following the standard few-shot adaptation benchmark. All datasets are public benchmarks used under their original train/test splits.

\paragraph{Setting B (Concat + MLP).} Features from selected encoders are concatenated (no normalization) and fed to a 2-layer MLP: $\text{Linear}(\text{input\_dim} \to 512) \to \text{ReLU} \to \text{Dropout}(0.1) \to \text{Linear}(512 \to \text{num\_classes})$. Training: Adam optimizer, $\text{lr} = 10^{-3}$, weight decay $10^{-4}$, batch size $128$, $20$ epochs.

\paragraph{Selection cost.}
KAGES needs only the centred encoder kernels and a greedy traversal, with no downstream training. Computing the ten $n \times n$ centred kernels takes $\sim$5s on a single GPU, and the full greedy ordering (three Frobenius inner products per candidate per step, $O(n^2)$ and independent of encoder dimension) completes in under one second. The total selection cost is a few seconds, against hours for exhaustive evaluation of all $2^{10} - 1 = 1023$ subsets.

\paragraph{Reproducibility.} Main-table experiments use five random seeds $\{17, 23, 42, 71, 137\}$ and report mean $\pm$ standard deviation. Paired method comparisons use $10{,}000$ bootstrap resamples for $95\%$ confidence intervals.

\subsection{Risk decomposition (Theorem~\ref{thm:risk}) and proof}
\label{app-y:proof-risk}

For a selected set $S$, let
\[
\bar F_S := (F_m)_{m\in S}
\]
denote the joint collection of its encoder features. Let $R^\star(S)$ be the Bayes risk given $\bar F_S$, and let $R_n(S)$ be the expected population test risk of the predictor returned by the $n$-sample training procedure. Define the finite-sample excess risk
\[
E_n(S):=R_n(S)-R^\star(S)\ge0.
\]
This exact excess term may absorb estimation, approximation, and optimization effects of the specified downstream training procedure.

\begin{theorem}[Risk decomposition]
\label{thm:risk}
For a nested growth path $S_0\subset S_1\subset\cdots\subset S_K$, the change in finite-sample risk at step $k$ decomposes as
\[
R_n(S_k)-R_n(S_{k-1})
=
\Delta E_k-\delta_k^\star,
\]
where
\[
\Delta E_k:=E_n(S_k)-E_n(S_{k-1}),
\qquad
\delta_k^\star:=R^\star(S_{k-1})-R^\star(S_k)\ge0.
\]
\end{theorem}

\begin{proof}
Subtract
$R_n(S_{k-1})=R^\star(S_{k-1})+E_n(S_{k-1})$
from
$R_n(S_k)=R^\star(S_k)+E_n(S_k)$.
The non-negativity of $\delta_k^\star$ follows because
$\bar F_{S_{k-1}}$ is a subcollection of $\bar F_{S_k}$: a Bayes decision rule given $\bar F_{S_k}$ can ignore the newly added feature and reproduce any rule based on $\bar F_{S_{k-1}}$. The exact increment $\Delta E_k$ need not be non-negative.
\end{proof}

\subsection{Marginal-utility crossing: proof of Theorem~\ref{main-thm:peak}}
\label{app-y:proof-peak}

This appendix gives the derivation behind Theorem~\ref{main-thm:peak} of the main paper. Along a fixed ordering, write
\[
g_k:=R^\star(S_{k-1})-R^\star(S_k),
\qquad
c_k:=E_n(S_k)-E_n(S_{k-1}),
\qquad
a_k:=g_k-c_k.
\]

\paragraph{Step 1: one-step identity.}
For every $k$,
\[
R_n(S_k)-R_n(S_{k-1})=c_k-g_k=-a_k.
\]
This follows immediately by subtracting the two risk decompositions in Theorem~\ref{thm:risk}.

\paragraph{Step 2: log loss makes $g_k$ a conditional mutual information.}
Under log loss $\ell(q,y)=-\log q(y)$, the Bayes minimizer is the true conditional distribution and
\[
R^\star_{\log}(S)=H(Y\mid\bar F_S).
\]
Since $\bar F_{S_k}=(\bar F_{S_{k-1}},F_{m_k})$,
\[
g_{k,\log}
=
I(F_{m_k};Y\mid\bar F_{S_{k-1}}).
\]
The chain rule gives
\begin{equation*}
I(F_{m_k};Y\mid\bar F_{S_{k-1}})
=
\underbrace{I(F_{m_k};Y)}_{\text{relevance}}
-
\underbrace{I(F_{m_k};\bar F_{S_{k-1}})}_{\text{redundancy}}
+
\underbrace{I(F_{m_k};\bar F_{S_{k-1}}\mid Y)}_{\text{class-conditional adjustment}}.
\end{equation*}

\paragraph{Step 3: a submodular greedy path has non-increasing Bayes gains.}
Let $f(S):=I(\bar F_S;Y)$ and suppose $f$ is submodular along the path. If
\[
m_k\in\arg\max_{m\notin S_{k-1}}\Delta f(m\mid S_{k-1}),
\]
then
\[
g_{k+1}
=
\max_{m\notin S_k}\Delta f(m\mid S_k)
\le
\max_{m\notin S_k}\Delta f(m\mid S_{k-1})
\le
g_k.
\]
Thus submodularity plus greedy selection is one sufficient mechanism for the non-increasing-$g_k$ assumption in Theorem~\ref{main-thm:peak}.

\paragraph{Step 4: Gaussian OLS instance with an increasing pure-noise tail.}
Let $X\sim\mathcal N(0,I_M)$,
$Y=\sum_{j=1}^{r}\beta_jX_j+\varepsilon$ with
$\varepsilon\sim\mathcal N(0,\sigma^2)$,
$F_{m_j}(X)=X_j$, and $S_k=\{m_1,\ldots,m_k\}$. For an OLS probe trained on $n>k+1$ samples,
\[
\mathbb E\,R_n(S_k)
=
\nu_k^2\frac{n-1}{n-k-1},
\qquad
\nu_k^2:=\sigma^2+\sum_{j=k+1}^{r}\beta_j^2.
\]
Here $R^\star(S_k)=\nu_k^2$, so
\[
E_n(S_k)=\nu_k^2\frac{k}{n-k-1}.
\]
For every pure-noise addition $k>r$, $\nu_{k-1}^2=\nu_k^2=\sigma^2$, and hence
\[
c_k
=
E_n(S_k)-E_n(S_{k-1})
=
\sigma^2\frac{n-1}{(n-k)(n-k-1)}
>0.
\]
This quantity is strictly increasing with $k$ while $n>k+1$. Thus the Gaussian model gives an explicit increasing-cost tail; it does not assert that $c_k$ is non-decreasing for arbitrary signal-bearing additions. As a concrete full-path instance, take $r=1$ and $n>K+1$. Then
$g_1=\beta_1^2$, $c_1=\sigma^2/(n-2)$, and for every $k\ge2$, $g_k=0<c_k$. Therefore $\beta_1^2>\sigma^2/(n-2)$ yields a strict single crossing after the first step.

\paragraph{Step 5: proof of the conditional crossing theorem.}
Under the assumptions of Theorem~\ref{main-thm:peak}, $g_k$ is non-increasing and $c_k$ is non-decreasing, so $a_k=g_k-c_k$ is non-increasing. Let
\[
k^\star:=\max\{k:a_k>0\}
\]
and let $\eta:=\min_k|a_k|>0$ denote the no-tie margin. Then
$a_k\ge\eta>0$ for $k\le k^\star$ and
$a_k\le-\eta<0$ for $k>k^\star$; the sign change occurs between $k^\star$ and $k^\star+1$. By Step~1,
$R_n(S_k)-R_n(S_{k-1})=-a_k$, which is strictly negative up to $k^\star$ and strictly positive afterwards. This proves the asserted order-specific peak-then-decline path.
\hfill$\square$

\subsection{Exact CKA margin derivation}
\label{app-y:exact-margin}

We give the full derivation of Proposition~\ref{main-prop:exact-margin} in the main paper. Recall the five scalars
$A_S := \langle K_S^c, K_Y^c \rangle_F$,
$B_S := \langle K_S^c, K_S^c \rangle_F$,
$r_m := \langle K_m^c, K_Y^c \rangle_F$,
$h_m := \langle K_m^c, K_S^c \rangle_F$,
$q_m := \langle K_m^c, K_m^c \rangle_F$,
with $C = \|K_Y^c\|_F^2$. This derivation concerns the pure alignment utility $U$, equivalently $J$ with $\tau=0$.

\paragraph{One-step CKA after adding $m$.}
Adding $m$ to $S$ gives $K_{S \cup \{m\}}^c = K_S^c + K_m^c$, so
\[
A_{S \cup \{m\}} = A_S + r_m,
\qquad
B_{S \cup \{m\}} = B_S + 2 h_m + q_m,
\qquad
U(S \cup \{m\}) = \frac{A_S + r_m}{\sqrt{(B_S + 2 h_m + q_m)\, C}}.
\]

\paragraph{Positive-gain condition.}
Assume $B_S>0$ and $C>0$. Since the centred Gram matrices are positive semidefinite, $A_S,r_m,h_m,q_m\ge0$. Then $U(S \cup \{m\}) > U(S)$ is equivalent to
$\frac{A_S + r_m}{\sqrt{B_S + 2 h_m + q_m}} > \frac{A_S}{\sqrt{B_S}}$,
which (with $A_S \ge 0$, both sides non-negative) is equivalent to
$(A_S + r_m)^2 B_S > A_S^2 (B_S + 2 h_m + q_m)$,
i.e.\ $r_m^2 B_S + 2 A_S r_m B_S > A_S^2 (2 h_m + q_m)$. Treating this as a quadratic in $r_m$ and taking the positive root yields the threshold $r_m^\star = A_S(\sqrt{1 + (2 h_m + q_m)/B_S} - 1)$, which is the exact criterion of Proposition~\ref{main-prop:exact-margin}.

\subsection{Diagnostic-protocol scaling tables}
\label{app-y:scaling-table}

Table~\ref{tab:y:scaling} reports the full per-dataset accuracy curves that underlie Figure~\ref{main-fig:overview} of the main paper.

\begin{table}[h]
\centering
\caption{10-shot accuracy (\%) as encoders are incrementally added along the literature-tracked path (CLIP\,$\to$\,DINOv2\,$\to$\,ConvNeXt\,$\to$\,SigLIP\,$\to$\,MAE\,$\to$\,ResNet-50; gated fusion). These are the per-dataset curves underlying Figure~\ref{main-fig:overview} and Table~\ref{main-tab:fulldata}. \textbf{Bold}: peak. $\star$: CLIP alone is optimal.}
\label{tab:y:scaling}
\small
\begin{tabular}{lcccccc c}
\toprule
Dataset & 1 & 2 & 3 & 4 & 5 & 6 & Peak \\
\midrule
CIFAR-100 & 60.29 & 80.11 & \textbf{80.37} & 80.18 & 80.20 & 79.92 & 3 \\
Pets & 85.04 & 92.86 & \textbf{93.75} & 93.46 & 93.65 & 93.49 & 3 \\
DTD & 64.41 & 72.67 & 73.52 & 74.05 & 74.17 & \textbf{74.18} & 6 \\
Country211$\star$ & \textbf{13.50} & 9.85 & 9.91 & 9.82 & 9.71 & 9.58 & 1 \\
GTSRB$\star$ & \textbf{61.15} & 51.54 & 52.73 & 52.34 & 51.97 & 52.58 & 1 \\
\bottomrule
\end{tabular}
\end{table}

\subsection{Two failure modes (formal characterization)}
\label{app-y:proof-failure-modes}

The risk decomposition (Theorem~\ref{thm:risk}) shows that step $k$ degrades performance exactly when the excess-risk increment $c_k=\Delta E_k$ exceeds the Bayes-risk reduction $g_k=\delta_k^\star$. Two empirically distinct mechanisms can make this net utility negative.

\subsubsection*{Type I: Task-conditional redundancy.}

\begin{definition}[Task-conditional redundancy]
\label{def:y:type1}
Encoder $m_k$ is \emph{task-conditionally redundant} with respect to $S_{k-1}$ for label $Y$ if
\[
I(F_{m_k};Y\mid\bar F_{S_{k-1}})=0.
\]
\end{definition}

\begin{proposition}[Type~I removes the Bayes gain]
\label{prop:y:type1}
Under log loss, if $m_k$ is task-conditionally redundant with respect to $S_{k-1}$, then $g_k=\delta_k^\star=0$ and
\[
R_n(S_k)-R_n(S_{k-1})=\Delta E_k.
\]
Consequently, the step degrades performance exactly when $\Delta E_k>0$. The same conclusion holds for any Bayes decision loss whenever $Y$ is conditionally independent of $F_{m_k}$ given $\bar F_{S_{k-1}}$.
\end{proposition}

\begin{proof}
Under log loss,
\[
g_k
=
H(Y\mid\bar F_{S_{k-1}})
-
H(Y\mid\bar F_{S_k})
=
I(F_{m_k};Y\mid\bar F_{S_{k-1}})
=0.
\]
The one-step identity then gives
$R_n(S_k)-R_n(S_{k-1})=\Delta E_k$.
\end{proof}

\subsubsection*{Type II: Task--model misalignment.}

\begin{definition}[Task--model misalignment]
\label{def:y:type2}
For a tolerance $\epsilon\ge0$, encoder $m_k$ is \emph{$\epsilon$-task-misaligned} relative to $S_{k-1}$ if its conditional Bayes gain satisfies
\[
g_k
=
R^\star(S_{k-1})-R^\star(S_k)
\le\epsilon.
\]
Under log loss this is equivalent to
$I(F_{m_k};Y\mid\bar F_{S_{k-1}})\le\epsilon$. The candidate may nevertheless be representationally distinct according to a task-agnostic diversity criterion.
\end{definition}

\begin{proposition}[Type~II decline when excess cost dominates]
\label{prop:y:type2}
If $m_k$ is $\epsilon$-task-misaligned relative to $S_{k-1}$ and $c_k>\epsilon$, then
\[
R_n(S_k)-R_n(S_{k-1})
=
c_k-g_k
\ge
c_k-\epsilon
>0.
\]
\end{proposition}

\begin{proof}
The claim follows directly from the one-step identity and the definition $g_k\le\epsilon$.
\end{proof}

\subsection{Image retrieval: full Recall@1 curves}
\label{app-y:retrieval-full}

With the same KAGES ordering used in the main text, L2-normalized concatenated features, and nearest-neighbor search, we report Recall@1 as a function of $k$ on three datasets (Table~\ref{tab:y:retrieval}).

\begin{table}[h]
\centering
\caption{Image retrieval Recall@1 (\%) vs.\ $k$. Retrieval gains are task-dependent and saturate at larger $k$ than classification.}
\label{tab:y:retrieval}
\small
\begin{tabular}{lccccc cc}
\toprule
Dataset & $k = 1$ & 3 & 5 & 8 & 10 & Peak $k$ & Best R@1 \\
\midrule
DTD & 47.4 & 66.6 & 70.1 & 70.5 & \textbf{74.6} & 10 & 74.6 \\
GTSRB & 50.4 & 56.8 & 55.7 & 59.2 & \textbf{60.2} & 10 & 60.2 \\
Pets & 89.8 & 93.3 & 93.5 & \textbf{94.3} & 94.0 & 8 & 94.3 \\
\bottomrule
\end{tabular}
\end{table}

\subsection{Fair 16-shot comparison with single-model adapters}
\label{app-y:vs-single}

Table~\ref{tab:y:vs-single} compares multi-model fusion with recent single-model CLIP adaptation methods under a fair 16-shot setting.

\begin{table}[h]
\centering
\caption{Fair 16-shot comparison: multi-model fusion vs.\ single-model CLIP adaptation (\%). All methods use 16 shots per class. $\dagger$: from original papers.}
\label{tab:y:vs-single}
\small
\begin{tabular}{llcccc}
\toprule
& Method & DTD & Flowers & GTSRB & Avg \\
\midrule
\multirow{5}{*}{\rotatebox{90}{\tiny Single}}
& CLIP zero-shot$\dagger$ & 44.3 & 67.4 & -- & -- \\
& CoOp$\dagger$~\citep{zhou2022coop} & 68.7 & 95.7 & -- & -- \\
& Tip-Adapter-F$\dagger$~\citep{zhang2022tipadapter} & 73.7 & 97.2 & -- & -- \\
& TaskRes$\dagger$~\citep{yu2023taskres} & 73.4 & 97.0 & -- & -- \\
& CaFo$\dagger$~\citep{zhang2023cafo} & 74.6 & 97.5 & -- & -- \\
\midrule
\multirow{3}{*}{\rotatebox{90}{\tiny Multi}}
& All Models & 77.9 & 99.6 & 61.6 & 79.7 \\
& Random & 77.7 & 99.6 & 60.3 & 79.2 \\
& \textbf{KAGES} & \textbf{78.7} & \textbf{99.7} & \textbf{72.1} & \textbf{83.5} \\
\bottomrule
\end{tabular}
\end{table}

\subsection{Full leaderboard on Protocol B}
\label{app-y:tables-full-b}

The main paper reports KAGES against the standard multi-view defaults (full fusion, random, diversity) and summarises the two stronger subset-selection baselines (DPP, submodular) in the text. Table~\ref{tab:y:full-b} gives the complete per-dataset AULC for all six multi-view methods on Protocol B (16-shot, Concat+MLP), means over five seeds. KAGES is best on every dataset and on the average; the coverage-driven DPP and submodular objectives are the strongest competitors but still trail KAGES, since neither uses the label kernel.

\begin{table}[h]
\centering
\caption{Full per-dataset AULC (\%) on Protocol B (16-shot, Concat+MLP), means over 5 seeds $\{17,23,42,71,137\}$. \textbf{Bold}: best per column; \underline{underlined}: second-best.}
\label{tab:y:full-b}
\small
\setlength{\tabcolsep}{6pt}
\begin{tabular}{l ccccc c}
\toprule
Method & C100 & Pets & DTD & C211 & GTSRB & Avg \\
\midrule
Fuse-All        & 78.54          & \underline{94.06} & 75.35          & 10.16          & 57.18          & 63.06 \\
Random          & 77.06          & 92.97          & 74.87          & 11.10          & 59.16          & 63.03 \\
Diversity       & 80.61          & 90.91          & 74.38          & 11.27          & 59.12          & 63.26 \\
DPP             & \underline{83.62} & 93.21          & 72.48          & \underline{11.87} & 64.88          & 65.21 \\
Submodular      & 82.82          & 93.31          & \underline{77.64} & 11.34          & \underline{65.53} & \underline{66.13} \\
\textbf{KAGES (ours)} & \textbf{83.84} & \textbf{94.70} & \textbf{77.82} & \textbf{12.16} & \textbf{66.12} & \textbf{66.93} \\
\bottomrule
\end{tabular}
\end{table}

\subsection{Normalization ablation}
\label{app-y:normalization-ablation}

The main paper uses the raw additive variant $K_S^c = \sum_{s\in S} K_s^c$. A per-encoder Frobenius-normalised variant $\widetilde{K}_s^c := K_s^c / (\|K_s^c\|_F + \varepsilon)$ gives provable per-encoder scale invariance but yields slightly lower empirical AULC on our pool ($\Delta \mathrm{AULC} \approx -0.4$\,pp on Protocol B). We report it here as an ablation and use the raw additive variant in the main results.

\subsection{Inter-seed stability of AULC}
\label{app-y:tables-full}

Across the five seeds $\{17,23,42,71,137\}$, the inter-seed standard deviation of AULC is small and does not affect the method ranking: it is at most $0.3$\,pp for KAGES on every dataset and below $0.6$\,pp for all baselines, so the per-dataset margins reported in the main tables are stable across seeds.


\section{Cross-Modal Validation on Frozen LLMs}
\label{app-y:llm}

To test whether KAGES transfers beyond vision, we run a parallel experiment on \textbf{6 frozen base LLMs}: Qwen2.5-7B~\citep{yang2024qwen25}, Mistral-7B-v0.1, Llama-3-8B~\citep{dubey2024llama3}, DeepSeek-7B, Gemma-7B, and Yi-1.5-9B. For each model we extract the last-layer, last-token hidden state on MMLU-500 and ARC-Challenge-500 (both 4-way MCQA, treated as classification), then apply the exact same pipeline as our main experiments: concatenate features of the first $k$ selected models, train a linear probe on an 80/20 split, and report accuracy (Table~\ref{tab:y:llm}).

\begin{table}[h]
\centering
\caption{Linear-probe accuracy (\%) across $k$ on frozen LLMs. Peak-then-decline reproduces on both tasks. KAGES peaks earlier and higher on MMLU.}
\label{tab:y:llm}
\small
\begin{tabular}{ll cccccc c}
\toprule
Dataset & Method & $k = 1$ & 2 & 3 & 4 & 5 & 6 & Peak \\
\midrule
\multirow{3}{*}{MMLU-500}
 & \textbf{KAGES} & 67.0 & \textbf{75.0} & 69.0 & 72.0 & 71.0 & 67.0 & \textbf{75.0} \\
 & diversity-greedy & 67.0 & 66.0 & 70.0 & 69.0 & 71.0 & 68.0 & 71.0 \\
 & fixed order & 67.0 & 68.0 & 72.0 & 70.0 & 69.0 & 67.0 & 72.0 \\
\midrule
\multirow{3}{*}{ARC-Ch-500}
 & \textbf{KAGES} & 88.0 & 86.0 & 89.0 & \textbf{91.0} & 88.0 & 87.0 & \textbf{91.0} \\
 & diversity-greedy & 88.0 & 86.0 & 89.0 & 90.0 & 89.0 & 87.0 & 90.0 \\
 & fixed order & 85.0 & 88.0 & 89.0 & 85.0 & 91.0 & 89.0 & 91.0 \\
\bottomrule
\end{tabular}
\end{table}

\begin{figure}[h]
\centering
\includegraphics[width=0.95\linewidth]{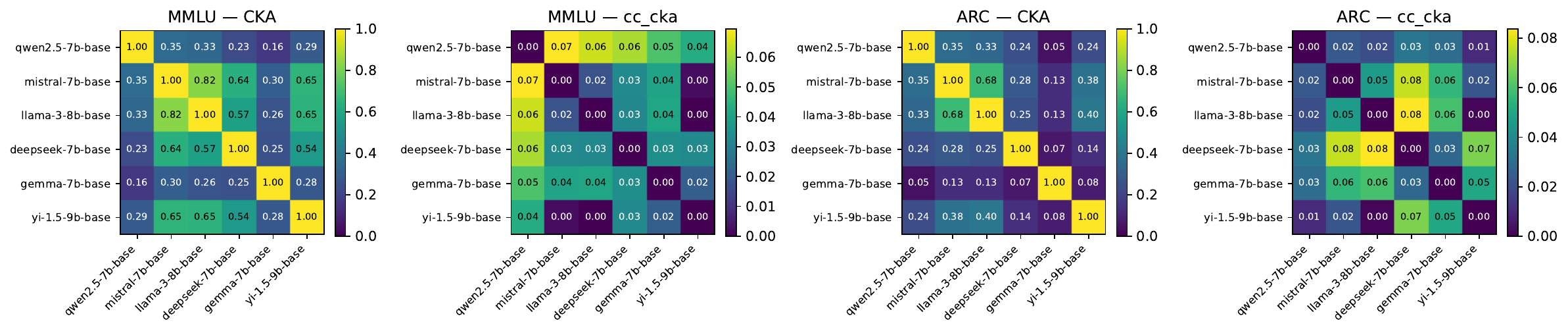}
\caption{\textbf{CKA matrices across 6 LLMs.} Mistral-Llama CKA reaches $0.82$ on MMLU (both English-heavy pretraining), while Gemma (Google, distilled) stays below $0.30$ with every other model. This echoes our vision finding: training corpus/objective, not architecture, drives representation similarity.}
\label{fig:y:llm-cka}
\end{figure}

\paragraph{Findings.}
(i) The peak-then-decline curve reproduces: MMLU peaks at $k = 2$ for KAGES ($75\%$) and at $k = 3$ to $5$ for the baselines, while ARC peaks at $k = 4$ to $5$. No configuration benefits from all 6 models. (ii) KAGES's three-term score beats diversity-greedy by up to $9$pp at $k = 2$ on MMLU, consistent with our main-paper observation that pure diversity is necessary but not sufficient. (iii) The CKA matrix supports the Platonic Representation Hypothesis~\citep{huh2024platonic} across modalities: Mistral and Llama, both English-dominant base models, cluster tightly, while Gemma's distillation-heavy training produces the most isolated representation.


\section{Future Work: Online and Sequential Extensions}
\label{app-y:future}

The present framework is offline: the labelled task is fixed at $n$ samples, the encoder pool is finite, and the selection is computed in a single pass. Two extensions are natural future directions.

\paragraph{Sequential / streaming labels.}
If labelled samples arrive sequentially, the ordering $\hpi$ and the optional score-selected size $\hat k_J$ can be recomputed as the empirical score path tightens. The sequential-analysis literature~\citep{tartakovsky2014sequential} provides change-point detection tools that could be adapted to detect significant shifts in the selection. Active best-arm identification~\citep{kaufmann2016complexity} provides sample-complexity templates for adaptive selection.

\paragraph{Online ordering and score-based size selection.}
The ordering and the optional score-selected size can be recomputed as new samples arrive. Under its stated assumptions, Theorem~\ref{thm:y:T1} implies stabilization of the \emph{ordering} after $O((B^2+\sigma^2[\log M+\log(1/\epsilon)])/G^2)$ samples with failure probability at most $\epsilon$; it does not by itself prove stabilization or accuracy optimality of $\hat k_J$. An online variant could expose intermediate orderings and score paths with confidence intervals. The connection to bandit best-arm identification~\citep{evendar2002pac,evendar2006action,kaufmann2016complexity} is natural and we leave the formal analysis to a follow-up.

\end{document}